\documentclass{article} 
\usepackage{iclr2022_conference,times}

\usepackage{amsmath,amsfonts,bm}

\def\eqref#1{equation~\ref{#1}}

\def\ceil#1{\lceil #1 \rceil}
\def\floor#1{\lfloor #1 \rfloor}
\def\1{\bm{1}}

\DeclareMathAlphabet{\mathsfit}{\encodingdefault}{\sfdefault}{m}{sl}
\SetMathAlphabet{\mathsfit}{bold}{\encodingdefault}{\sfdefault}{bx}{n}

\DeclareMathOperator*{\argmax}{arg\,max}
\DeclareMathOperator*{\argmin}{arg\,min}

\usepackage{hyperref}
\usepackage{url}
\usepackage{amsmath,amssymb,algorithm,algorithmicx,algpseudocode,amsthm,bm,bbm}

\newtheorem{assumption}{Assumption}
\newtheorem{theorem}{Theorem}
\newtheorem{lemma}{Lemma}

\newtheorem{definition}{Definition}

\title{Near-Optimal Reward-Free Exploration for Linear Mixture MDPs with Plug-in Solver}

\author{Xiaoyu Chen \& Jiachen Hu  \\
Key Laboratory of Machine Perception, MOE, \\School of Artificial Intelligence,
Peking University \\
\texttt{\{cxy30, NickH\}@pku.edu.cn} \\
\And
Lin F. Yang  \thanks{Corresponding author.} \\ 
Electrical and Computer Engineering Department, \\
University of California, Los Angeles \\
\texttt{linyang@ee.ucla.edu} \\
\And
Liwei Wang \textsuperscript{*}\\
Key Laboratory of Machine Perception, MOE,\\ School of Artificial Intelligence,
Peking University \\
International Center for Machine Learning Research, Peking University \\
\texttt{wanglw@cis.pku.edu.cn}
}

\iclrfinalcopy 
\begin{document}

\maketitle

\begin{abstract}
Although model-based reinforcement learning (RL) approaches are considered more sample efficient, existing algorithms are usually relying on sophisticated planning algorithm to couple tightly with the model-learning procedure. Hence the learned models may lack the ability of being re-used with more specialized planners. In this paper we address this issue and provide approaches to learn an RL model efficiently without the guidance of a reward signal. In particular, we take a plug-in solver approach, where we focus on learning a model in the exploration phase and demand that \emph{any planning algorithm} on the learned model can give a near-optimal policy. specifically, we focus on the linear mixture MDP setting, where the probability transition matrix is a (unknown) convex combination of a set of existing models. We show that, by establishing a novel exploration algorithm, the plug-in approach learns a model by taking $\tilde{O}(d^2H^3/\epsilon^2)$ episodes with the environment and \emph{any} $\epsilon$-optimal planner on the model gives an $O(\epsilon)$-optimal policy on the original model. This sample complexity matches our lower bound for non-plug-in approaches and is \emph{statistically optimal}. We achieve this result by leveraging a careful maximum total-variance bound using Bernstein inequality and properties specified to linear mixture MDPs.
\end{abstract}

\section{Introduction}
\label{sec:introduction}
In reinforcement learning, an agent repeatedly interacts with the unknown environment in order to maximize the cumulative reward. To achieve this goal, an RL algorithm must be equipped with effective exploration mechanisms to learn the unknown environment and
find a near-optimal policy. 
Efficient exploration is critical to the success of reinforcement learning algorithms, which has been widely investigated from both the empirical and the theoretical perspectives (e.g.~\cite{stadie2015incentivizing,pathak2017curiosity,azar2017minimax,jin2018q}).
Model-based RL is one of the important approaches to solve for the RL environment. In model-based RL, the agent learns the model of the environment and then performs planning in the estimated model. It  has been widely applied in many RL scenarios, including both online setting~\citep{kaiser2019model,luo2019algorithmic,azar2017minimax} and offline setting~\citep{yu2020mopo,kidambi2020morel}. It is also believed that model-based RL is significantly more sample-efficient than model-free RL, which has been justified by many recent empirical results (e.g.~\cite{kaiser2019model,wang2019benchmarking}). Though the theoretical model-based learning in small scale problems has been studied extensively \citep{azar2017minimax, zhou2020nearly, jin2020reward}, it is still far from complete, especially with the presence of a function approximator. 

As an important implication of model-based approaches, the power of \textit{plug-in} approach have been studied in several works \citep{cui2020plug, agarwal2020model1}. The idea of plug-in approach is rather simple: We construct an empirical Markov Decision Process (MDP) using maximum likelihood estimate, then return the (approximate) optimal policy with \textit{efficient planning algorithm} in this empirical model. The significance of plug-in approaches is two-folded. For one thing, it preserves an empirical model that keeps the value of the policies, which is of independent interests. For another, the empirical model can be used for any down-stream tasks, which makes the application much more flexible. It is shown that the plug-in approach achieves the minimax sample complexity to compute the $\epsilon$-optimal policies with a generative model in the tabular \citep{ agarwal2020model1} and linear settings \citep{cui2020plug}. 

In this paper, we aim to understand the power of plug-in approach in the reward-free exploration with linear function approximation. We study the linear mixture MDPs, where the transition probability kernel is a linear mixture of a number of basis kernels~\citep{ayoub2020model,zhou2020nearly,zhou2020provably}. 
 We first build an empirical model with an estimation of the transition dynamics in the exploration phase, and then find a near-optimal policy by planning with the empirical model via arbitrary plug-in solver in the planning phase. 
Our setting is different from the reward-free exploration with linear function approximation without plug-in model~\citep{wang2020reward,zanette2020provably}, in which the agent can directly observe all history samples and design specialized model-free algorithm in the planning phase. 


Our results show that the plug-in approach can achieve near-optimal sample complexity in the reward-free setting. In particular, we proposed a statistically efficient algorithm for reward-free exploration. Our algorithm samples $\tilde{O}(d^2H^4/\epsilon^2)$ trajectories during the exploration phase, which suffices to obtain $O(\epsilon)$-optimal policies for an arbitrary reward function with an $\epsilon$-optimal pluging solver in the planning phase. Here $d$ is the feature dimension, and $H$ is the planning horizon. Furthermore, with a more refined trajectory-wise uncertainty estimation, we further improve the sample complexity bound to $\tilde{O}\left(d^2H^3/\epsilon^2\right)$ in the regime where $d > H$ and $\epsilon \leq H/\sqrt{d}$. This matches our lower bound $\Omega(d^2H^3/\epsilon^2)$ for reward-free exploration in linear mixture MDPs, which indicates that our upper bound is near-optimal except for logarithmic factors. To the best of our knowledge, this is the first work that obtains minimax sample complexity bounds for the plug-in approach in reward-free exploration with linear function approximation. 

\section{Related Work}
\label{sec: related work}
\paragraph{RL with Linear Function Approximation}
Reinforcement learning with linear function approximation has been widely studied in the recent few years (e.g.~\cite{jiang2017contextual,yang2019sample,yang2020reinforcement,jin2020provably,modi2020sample,du2019good,zanette2020learning,cai2020provably,ayoub2020model,weisz2021exponential,zhou2020nearly,zhou2020provably}). The linear mixture MDPs model studied in our work assumes the transition probability function is parameterized as a linear function of a given feature mapping over state-action-next-state triple~\citep{ayoub2020model,zhou2020provably,zhou2020nearly}. Based on the Bernstein inequality for vector-valued martingales, \cite{zhou2020nearly} proposed an efficient algorithm that obtains minimax regret in the regime where $d > H$. Besides linear mixture MDPs, linear MDPs is another category of RL with linear function approximation, which assumes both the transition probability function and reward function are parameterized as a linear function of a given feature mapping over state-action pairs. The algorithms with best regret bounds were proposed by \cite{jin2020provably} and \cite{yang2020reinforcement}, which studied model-free algorithm and model-based algorithm respectively. The minimax regret bound for linear MDPs is still unclear.

\paragraph{Reward-Free Reinforcement Learning}
In contrast to the standard RL setting, reward-free reinforcement learning  separates the exploration problem and the planning problem, which allows one to handle them in a theoretically principled way. For tabular setting, reward-free reinforcement learning has been well-exploited in many previous results~\citep{jin2020reward,kaufmann2021adaptive,menard2020fast,zhang2020task,zhang2021near,wu2021gap,bai2020provable,liu2021sharp}, where the minimax rate is obtained by \cite{menard2020fast}. For reward-free exploration with linear function approximation, \cite{wang2020reward} proposed the first efficient algorithm that obtains $O(d^3H^6/\epsilon^2)$ sample complexity for linear MDPs. However, their algorithm is model-free in nature and cannot guarantee good performance with any plug-in solver. 
Further, \cite{qiu2021reward} proposed the first provably efficient reward-free algorithm with kernel and neural function approximation.

We also noticed that there is a concurrent work which also studied reward-free exploration for linear mixture MDPs~\citep{zhang2021reward}. Compared with their results, we focus on the setting of the reward-free exploration with plug-in solver, which covers the standard reward-free setting  studied in the previous results. Furthermore, our sample complexity bounds are tighter than theirs by a factor of $H^2$ \footnote{When transformed to the time-homogeneous MDPs setting studied in~\cite{zhang2021reward}, our algorithms can achieve sample complexity bounds $\tilde{O}(d^2H^3/\epsilon^2)$ and $\tilde{O}((d^2H^2+dH^3)/\epsilon^2)$, respectively.}. The above two differences introduce new challenges in both the algorithmic design and the complexity analysis in this work, which makes our algorithms much more complicated than theirs. Besides, our lower bound is tighter than theirs in the dependence on $d$.

\paragraph{Plug-in Approach} The plug-in approach has been studied in tabular/linear case in restrictive settings.
E.g., \citet{agarwal2020model1} and \citet{cui2020plug}  studied the standard plug-in approach with a generative model, where the algorithm is allowed to query the outcome of any state action pair from an oracle. They showed the plug-in approach also achieved the minimax optimal sample complexity to find an $\epsilon$-optimal policy in both tabular MDPs and linear MDPs. The reward-free algorithms proposed by \cite{jin2020reward} are model-based in nature, thus can be regarded as a solution in the plug-in solver setting. 
However, their algorithms are restricted to the tabular case and cannot be applied to the setting with linear function approximation.
\section{Preliminaries}
\subsection{Episodic MDPs}
We consider the setting of episodic Markov decision processes (MDPs), which can be denoted by a six-tuple $(\mathcal{S},\mathcal{A},P,R,H,\nu)$, where $\mathcal{S}$ is the set of states, $\mathcal{A}$ is the action set, $P$ is the transition probability matrix so that $P_h(\cdot|s,a)$ gives the distribution over states if action $a$ is taken on state $s$ at step $h$, $R_h(s,a)$ is the deterministic reward function of taking action $a$ on state $s$ with support $[0,1]$ in step $h$, $H$ is the number of steps in each episode, and $\nu$ is the distribution of the initial state.

In episode $k$, the agent starts from an initial state $s_{k,1}$ sampled from the distribution $\nu$. At each step $h \in [H]$, the agent observes the current state $s_{k,h} \in \mathcal{S}$, takes action $a_{k,h} \in \mathcal{A}$, receives reward $R_h(s_{k,h},a_{k,h})$, and transits to state $s_{k,h+1}$ with probability $P_h(s_{k,h+1}|s_{k,h},a_{k,h})$. The episode ends when $s_{H+1}$ is reached.

A deterministic policy $\pi$ is a collection of $H$ policy functions $\{\pi_h: \mathcal{S} \rightarrow \mathcal{A}\}_{h \in [H]}$. We use $\Pi$ to denote the set of all deterministic policies. For a specific reward function $R$, we use $V_h^{\pi}: \mathcal{S} \times R \rightarrow \mathbb{R}$ to denote the value function at step $h$ under policy $\pi$ w.r.t. reward $R$, which gives the expected sum of the remaining rewards received under policy $\pi$ starting from $s_h = s$, i.e. $V_{h}^{\pi}(s,R)=\mathbb{E}\left[\sum_{h^{\prime}=h}^{H} R\left(s_{h^{\prime}}, \pi_{h^{\prime}}\left(s_{h^{\prime}}\right)\right) \mid s_{h}=s, P\right].$
Accordingly, we define $Q_h^{\pi}(s,a,R)$ as the expected Q-value function at step $h$: $Q^{\pi}_h(s,a,R) = \mathbb{E}\left[R\left(s_{h}, a_{h}\right)+\sum_{h^{\prime}=h+1}^{H} R\left(s_{h^{\prime}}, \pi_{h^{\prime}}\left(s_{h^{\prime}}\right)\right) \mid s_{h}=s, a_{h}=a, P\right].$

We use $\pi^*_R$ to denote the optimal policy w.r.t. reward R, and we use $V^*_h(\cdot,R)$ and $Q^*_h(\cdot,\cdot,R)$ to denote the optimal value and Q-function under optimal policy $\pi^*_R$ at step $h$. We say a policy $\pi$ is $\epsilon$-optimal w.r.t. reward $R$ if
$\mathbb{E}\left[\sum_{h=1}^{H} R_h\left(s_{h}, a_{h}\right) \mid \pi\right] \geq \mathbb{E}\left[\sum_{h=1}^{H} R_h\left(s_{h}, a_{h}\right) \mid \pi_{R}^{*}\right]-\epsilon.$

For the convenience of explanation, we assume the agent always starts from the same state $s_1$ in each episode. It is straightforward to extend to the case with stochastic initialization, by adding a initial state $s_0$ with no rewards and only one action $a_0$, and the transition probability of $(s_0,a_0)$ is the initial distribution $\mu$. We use $P_h V(s,a,R)$ as a shorthand of $\sum_{s'} P_h(s'|s,a) V(s',R)$.

\subsection{Linear Mixture MDPs}
We study a special class of MDPs called linear mixture MDPs, where the transition probability kernel is a linear mixture of a number of basis kernels~\citep{ayoub2020model,zhou2020provably,zhou2020nearly}. This model is defined as follows in the previous literature.



\begin{definition}
\citep{ayoub2020model} Let $\phi(s,a,s'): \mathcal{S}\times\mathcal{A} \times \mathcal{S} \rightarrow \mathbb{R}^d$ be a feature mapping satisfying that for any bounded function $V: \mathcal{S} \rightarrow [0,1]$ and any tuple $(s,a) \in \mathcal{S} \times \mathcal{A}$, we have $\left\|\phi_{V}(s, a)\right\|_{2} \leq 1, \text { where } \phi_{V}(s, a)=\sum_{s^{\prime} \in \mathcal{S}} \phi\left( s, a,s^{\prime}\right) V\left(s^{\prime}\right).$ An MDP is called a linear mixture MDP if there exists a parameter vector $\theta_h \in \mathbb{R}^d$ with $\|\theta_h\|_2 \leq B$ for a constant $B$ and a feature vector $\phi(\cdot,\cdot,\cdot)$, such that $P_h(s'|s,a) = \theta_h^{\top} \phi(s,a,s')$ for any state-action-next-state triplet $(s,a,s') \in \mathcal{S} \times \mathcal{A} \times \mathcal{S}$ and step $h \in [H]$.
\end{definition}

\subsection{Reward-Free Reinforcement Learning}
We study the problem of reward-free exploration with plug-in solver. Our setting is different from the reward-free exploration setting studied in the previous literature. Formally, there are two phases in this setting: exploration phase and planning phase.

During the exploration phase,  the agent interacts with the environment for $K$ episodes. In episode $k$, the agent chooses a policy $\pi_k$ which induces a trajectory. The agent observes the states and actions $s_{k,1},a_{k,1}, \cdots, s_{k,H},a_{k,H}$ as usual, but does not observe any rewards. After $K$ episodes, the agent calculates the estimated model $\{\tilde{P}_h = \tilde{\theta}_h^{\top} \phi\}_{h \in [H]}$, which will be used in the planning phase to calculate the optimal policy.

During the planning phase, the agent is no longer allowed to interact with the MDP. Also, it cannot directly observe the history samples obtained in the exploration phase. Instead, the agent is given a set of reward function $\{R_h\}_{h \in [H]}$, where $R_h: \mathcal{S} \times \mathcal{A} \rightarrow [0,1]$ is the deterministic reward in step $h$. For notation convenience, we occasionally use $R$ as a shorthand of $\{R_h\}_{h \in [H]}$ during the analysis. We define $\hat{V}_{h}^{\pi,\tilde{P}}(s,R)$ as the value function of transition $\tilde{P}$ and reward $R$, i.e. $\hat{V}_{h}^{\pi,\tilde{P}}(s,R)=\mathbb{E}\left[\sum_{h^{\prime}=h}^{H} R_h\left(s_{h^{\prime}}, \pi_{h^{\prime}}\left(s_{h^{\prime}}\right)\right) \mid s_{h}=s, \tilde{P}\right].$

In the planning phase, the agent calculates the optimal policy $\hat{\pi}_R$ with respect to the reward function $R$ in the estimated model $\tilde{P}$ using any $\epsilon_{\rm opt}$-optimal model-based solver. That is, the returned policy $\hat{\pi}_R$ satisfies:
$\hat{V}^{*,\tilde{P}}_{1}(s_1,R) - \hat{V}_1^{\hat{\pi}_R,\tilde{P}}(s_1,R) \leq \epsilon_{\rm opt}.$

The agent's goal is to output an accurate model estimation $\tilde{P}$ after the exploration phase, so that the policy $\hat{\pi}_R$ calculated in planning phase can be $\epsilon+\epsilon_{\rm opt}$-optimal w.r.t. any reward function $R$. 
Compared with the reward-free setting studied in the previous literature~\citep{jin2020reward,wang2020reward,menard2020fast,kaufmann2021adaptive}, the main difference is that we require that the algorithm maintains a model estimation instead of all the history samples after the exploration phase, and can use any model-based solver to calculate the near-optimal policy in the planning phase.
\section{Reward-Free RL with Plug-in Solver}
\label{sec: reward-free algorithm}




\subsection{Algorithm}
\label{subsec: Hoeffding algorithm}

\begin{algorithm}
\caption{Reward-free Exploration: Exploration Phase}
\label{alg: exploration phase}
  \begin{algorithmic}[5]
  \State Input: Failure probability $\delta>0$ and target accuracy $\epsilon > 0$
  \State $\lambda \leftarrow B^{-2}$, $\beta \leftarrow H\sqrt{d\log(4H^3K\lambda^{-1}\delta^{-1})} + \sqrt{\lambda}B$, $\mathcal{V} \leftarrow \{V: \mathcal{S} \rightarrow [0,H]\}$
    \For { episode $k = 1,2,\cdots, K$}
        \State ${Q}_{k,H+1}(\cdot,\cdot) = 0$, ${V}_{k,H+1}(\cdot) = 0$
        \For{step $h=H,H-1,\cdots, 1$}
            \State ${\Lambda}_{k,h} \leftarrow \sum_{t=1}^{k-1}  {\phi}_{t,h}(s_{t,h},a_{t,h}) {\phi}_{t,h}(s_{t,h},a_{t,h})^{\top} + \lambda I$
            \State $\hat{ {\theta}}_{k,h} \leftarrow \left( {\Lambda}_{k,h}\right)^{-1} \sum_{t=1}^{k-1}  {\phi}_{t,h}(s_{t,h},a_{t,h}) \tilde{V}_{t,h+1,s,a}(s_{t,h+1}) $
            \State $\tilde{V}_{k,h+1,s,a} \leftarrow \argmax_{V\in \mathcal{V}}\left\|\sum_{s'} {\phi}(s,a,s')V(s')\right\|_{({ {\Lambda}}_{k,h})^{-1}}, \forall s,a$
            \State $ {\phi}_{k,h}(s,a) \leftarrow \sum_{s'} {\phi}(s,a,s')\tilde{V}_{k,h+1,s,a}(s'), \forall s,a$
            \State $u_{k,h}(s, a) \leftarrow \beta  \sqrt{ {\phi}_{k,h}(s, a)^{\top}\left( {\Lambda}_{k,h}\right)^{-1}  {\phi}_{k,h}(s, a)}, \forall s,a$
            \State Define the exploration-driven reward function $R_{k,h}(s,a) = u_{k,h}(s,a), \forall s,a$
            \State $Q_{k,h}(s, a) \leftarrow \min \left\{\hat{ {\theta}}_{k,h}^{\top} \left(\sum_{s'} {\phi}(s, a,s')V_{k,h+1}(s')\right)+R_{k,h}(s, a)+u_{k,h}(s, a), H\right\}$ 
            \State $V_{k,h}(s)\leftarrow \max_{a\in \mathcal{A}} Q_{k,h}(s,a)$, $\pi_{k,h}(s) = \argmax_{a \in \mathcal{A}} Q_{k,h}(s,a)$
        \EndFor
        \For{step $h = 1,2,\cdots, H$}
            \State Take action $a_{k,h} = \pi_{k,h}(s_{k,h})$ and observe $s_{k,h+1} \sim P_h(s_{k,h},a_{k,h})$
        \EndFor
    \EndFor
    \State Find $\tilde{ {P}}_h$ such that the transition $\tilde{P}_h(\cdot|\cdot,\cdot) = \tilde{\theta}_h^{\top} \phi (\cdot,\cdot,\cdot)$ is well-defined and $\left\|\tilde{ {\theta}}_h-\hat{ {\theta}}_{K,h}\right\|_{ {\Lambda}_{K,h}} \leq \beta $ for $h \in [H]$
    \State Output: $\{\tilde{ {P}}_h\}_{h=1}^{H}$
  \end{algorithmic}
\end{algorithm}

The exploration phase of the algorithm is presented in Algorithm~\ref{alg: exploration phase}. Recall that for a given value function $\tilde{V}_{k,h+1}$, we have ${P}_h\tilde{V}_{k,h+1}(s_{k,h},a_{k,h}) = \theta_h^{\top}\phi_{\tilde{V}_{k,h+1}}(s_{k,h},a_{k,h}) $ for any $k,h$. Therefore, $\tilde{V}_{k,h+1}(s_{k,h+1})$ and $\phi_{\tilde{V}_{k,h+1}}(s_{k,h},a_{k,h})$ can be regarded as the stochastic reward and the linear feature of a linear bandits problem with linear parameter $\theta_h$. We employ the standard least-square regression to learn the underlying parameter $\theta_h$.
In each episode, we first update the estimation $\hat{\theta}_{k}$ based on history samples till episode $k-1$. We define the auxiliary rewards $R_{k,h}$ to guide exploration. We calculate the optimistic Q-function using the parameter estimation $\hat{\theta}_{k}$, and then execute the greedy policy with respect to the updated Q-function to collect new samples. 

The main problem is how to define the exploration-driven reward $R_{k,h}(s,a)$, which measures the uncertainty for state-action pair $(s,a)$ at the current step. In the setting of linear mixture MDPs, the linear feature $\phi_V(s,a)$ is a function of both state-action pair $(s,a)$ and the next-step value function $V$. This is not a big deal in the standard exploration setting~\citep{ayoub2020model,zhou2020nearly}. However, in the reward-free setting, since the reward function is not given beforehand, we need to upper bound the estimation error of value functions for \textit{any} possible rewards. 
To tackle this problem, we use $\max_{V \in \mathcal{V}}\beta \|\phi_V(s,a)\|_{\Lambda^{-1}_{k,h}}=\max_{V \in \mathcal{V}}\beta  \sqrt{ {\phi}_{V}(s, a)^{\top}\left( {\Lambda}_{k,h}\right)^{-1}  {\phi}_{V}(s, a)}$ as a measure of the maximum uncertainty for the state-action pair $(s,a)$, where $\mathcal{V} = \{V: \mathcal{S} \rightarrow [0,H]\}$ is the set of all possible value functions, and $\Lambda_{k,h}$ is the summation of all the history samples $\{s_{t,h},a_{t,h},s_{t,h+1}\}_{t=1}^{k-1}$ with feature $\phi_{t,h}(s_{t,h},a_{t,h}) = \argmax_{\phi_{V}}\beta \|\phi_V(s,a)\|_{\Lambda^{-1}_{t,h}}$.  In each episode, we define $R_{k,h}(s,a) = u_{k,h}(s,a) = \max_{V \in \mathcal{V}}\beta \|\phi_V(s,a)\|_{\Lambda^{-1}_{k,h}}$, where $R_{k,h}(s,a)$ is the exploration-driven reward used to guide exploration, and $u_{k,h}(s,a)$ is the additional bonus term which helps to guarantee that $Q_{k,h}(s,a)$ is an optimistic estimation. 
Finally, the algorithm returns the model estimation $\{\tilde{P}_h\}_h$ that is well defined (i.e. $\sum_{s'} \tilde{P}_h(s'|s,a) = 1, \tilde{P}_h(s'|s,a) \geq 0, \forall s,a,s'$), and satisfy the constraints $\left\|\tilde{ {\theta}}_h-\hat{ {\theta}}_{K,h}\right\|_{ {\Lambda}_{K,h}} \leq \beta $.



\subsection{Implementation Details}
\label{appendix: implementation details}
Algorithm~\ref{alg: exploration phase} involves two optimization problems in line 8 and line 19. These problems can be formulated as the standard convex optimization problem with a slight modification.
Specifically, the optimization problem in line 8 of Algorithm~\ref{alg: exploration phase} can be formulated in the following way:
\begin{align}
    \max_{V}\left\|\sum_{s'} {\phi}(s,a,s')V(s')\right\|_{({ {\Lambda}}_{k,h})^{-1}} \quad s.t. \quad 0 \leq V(s) \leq H, \forall s \in \mathcal{S} 
\end{align}

In general, solving this optimization problem is hard. For the case of finite state space ($S \leq \infty$), a recent work of \cite{zhang2021reward} relaxed the problem to the following linear programming problem:
\begin{align}
    \max_{\mathbf{f}}\left\|\Sigma_{1, k}^{-1 / 2} \boldsymbol{\Phi}(s, a) \mathbf{f}\right\|_1 \quad s.t. \quad \|\mathbf{f}\|_{\infty} \leq H,
\end{align}
where $\boldsymbol{\Phi}(s, a)=\left(\boldsymbol{\phi}\left(s, a, S_{1}\right), \cdots, \boldsymbol{\phi}\left(s, a, S_{|\mathcal{S}|}\right)\right)$ and $\mathbf{f}=\left(f\left(S_{1}\right), \cdots, f\left(S_{|\mathcal{S}|}\right)\right)^{\top}$. As discussed in \cite{zhang2021reward}, the sample complexity will be worse by a factor of $d$ if we solve the linear programming problem as an approximation. For the case where the state apace is infinite, we can use state aggregation methods~\citep{ren2002state,singh1995reinforcement} to reduce the infinite state space to finite state space and apply the approximation approaches to solve it.

The optimization problem in line 19 of Algorithm~\ref{alg: exploration phase} is to find parameter $\tilde{\theta}_h$ satisfying several constraints. For the case where the state space is finite, we can solve this problem in the following way:
\begin{align*}
    \min_{\tilde{\theta}_h} \left\|\tilde{ {\theta}}_h-\hat{ {\theta}}_{K,h}\right\|^2_{ {\Lambda}_{K,h}}  \quad
    s.t. \quad \sum_{s'} \tilde{\theta}^{\top} \phi(s,a,s') = 1,
    \tilde{\theta}^{\top} \phi(s,a,s') \geq 0, \forall s,a \in \mathcal{S} \times \mathcal{A}
\end{align*}

The above problem can be regarded as a quadratic programming problem and can be solved efficiently by the standard optimization methods. By Lemma~\ref{lemma: confidence set for theta}, we know that the true parameter ${\theta}_h$ satisfies $\|\hat{ {\theta}}_{K,h} -  {\theta}_h \|_{ {\Lambda}_{K,h}} \leq \beta$ with high probability. Therefore, the solution $\tilde{\theta}_{h}$ satisfies the constraint $\left\|\tilde{ {\theta}}_h-\hat{ {\theta}}_{K,h}\right\|_{ {\Lambda}_{K,h}} \leq \beta$ with high probability.

For the case where the state space is infinite, we can also solve the above problem using state aggregation methods. In particular, if the linear mixture MDP model can be regarded as a linear combination of several base MDP models (i.e. $\phi(s,a,s') = \left(P_1(s'|s,a), P_2(s'|s,a), \cdots, P_d(s'|s,a)\right)^{\top}$ where $P_i(s'|s,a)$ is the transition probability of certain MDP model), then we can formulate the optimization problem in the following way:
\begin{align*}
    \min_{\tilde{\theta}_h} &\left\|\tilde{ {\theta}}_h-\hat{ {\theta}}_{K,h}\right\|^2_{ {\Lambda}_{K,h}}  \quad
    s.t. \quad \tilde{\theta}^{\top} \boldsymbol{1} = 1, \tilde{\theta} \succeq 0,
\end{align*}
which can also be solved efficiently in the case of infinite state space.

\subsection{Regret}
\label{subsec: Hoeffding analysis}

\begin{theorem}
\label{theorem: main}
With probability at least $1-\delta$, after collecting $K=\tilde{O}\left(\frac{d^2 H^4}{\epsilon^2}\right)$ trajectories , Algorithm~\ref{alg: exploration phase} returns a transition model $\tilde{P}$, then for any given reward in the planning phase, a policy returned by any $\epsilon_{\rm opt}$-optimal plug-in solver on $(S, A, \tilde{P}, R,H,\nu)$ is $O(\epsilon + \epsilon_{\rm opt})$-optimal for the true MDP, $(S, A, P, R,H,\nu)$.
\end{theorem}

We also propose a lower bound for reward-free exploration in Appendix~\ref{appendix: lower bound}. Our lower bound indicates that $\Omega(d^2H^3/\epsilon^2)$ episodes are necessary to find an $\epsilon$-optimal policy with constant probability. 
This lower bound is achieved by connecting the sample complexity lower bound with the regret lower bound of certain constructed learning algorithms in the standard online exploration setting. Compared with this bound, our result matches the lower bound w.r.t. the dimension $d$ and the precision $\epsilon$ except for logarithmic factors.

There is also a recent paper of~\cite{wang2020reward} studying reward-free exploration in the setting of linear MDPs. Their sample complexity is $\tilde{O}(d^3H^6/\epsilon^2)$, thus our bound is better than theirs by a factor of $H^2$. 
Though the setting is different, we find that our parameter choice and the more refined analysis is applicable to their setting, which can help to further improve their bound by a factor of $H^2$. Please see Appendix~\ref{appendix: sample complexity, linear MDP} for the detailed discussion. 

\section{Improving the dependence on $H$}
\label{sec: Berstein algorithm}

In this section, we close the gap on $H$ with a maximum total-variance bound using Bernstein inequality.
In the previous results studying regert minimization in online RL setting~\citep{azar2017minimax, zhou2020nearly,wu2021nearly}, one commonly-used approach to obtain the minimax rate is to upper bound the regret using the total variance of the value function by the Bernstein's concentration inequalities, and finally bound the summation of the one-step transition variance by the law of total variance~\citep{lattimore2012pac,azar2013minimax,azar2017minimax}. However, the situation becomes much more complicated in the reward-free setting with linear function approximation. 
Recall that $\tilde{V}_{k,h+1,s,a}(s')$ defined in Line 9 of Algorithm~\ref{alg: exploration phase} is the next-step value function that maximizes the uncertainty $\|\phi_V(s,a)\|_{\Lambda^{-1}_{k,h}}$ for state-action pair $(s,a)$. One naive approach is to still use $\max_{V}\|\phi_V(s,a)\|_{\Lambda^{-1}_{k,h}}$ as the uncertainty measure for $(s,a)$ and upper bound the error rate by the summation of the one-step transition variance of $\tilde{V}_{k,h+1,s_{k,h},a_{k,h}}(s')$. However, we can not upper bound the variance summation of $\tilde{V}_{k,h+1,s_{k,h},a_{k,h}}(s')$ in $H$ steps by $O(H^2)$ similarly by the law of total variance, since $\{\tilde{V}_{k,h,s_{k,h},a_{k,h}}\}_{h \in [H]}$ is not the value functions induced from the same policy and transition dynamics. To tackle the above problem, we need to define the exploration-driven reward and the confidence bonus for each state-action pair in a more refined way. Our basic idea is to carefully measure the expected total uncertainty along the whole trajectories w.r.t. each MDPs in the confidence set. We will explain the detail in the following subsections. 

\subsection{Algorithm}

\begin{algorithm}
\caption{Reward-free Exploration: Exploration Phase}
\label{alg: exploration phase, Bernstein}
  \begin{algorithmic}[5]
  \State Input: Failure probability $\delta>0$ and target accuracy $\epsilon > 0$
  \State $\lambda \leftarrow B^{-2}$, $\hat{\beta}\leftarrow 16\sqrt{d\log(1+KH^2/(d\lambda))\log(32K^2H/\delta)} + \sqrt{\lambda} B$
  \State $\check{\beta}\leftarrow16d\sqrt{\log(1+KH^2/(d\lambda))\log(32K^2H/\delta)}+ \sqrt{\lambda} B$  
  \State $\tilde{\beta}\leftarrow16H^2\sqrt{d\log(1+KH^4/(d\lambda))\log(32K^2H/\delta)}+ \sqrt{\lambda} B$
  \State Set $\Lambda_{i,k,h} \leftarrow \lambda{I}$, $\hat{\theta}_{i,k,h} \leftarrow \boldsymbol{0}$ for $k=1, h \in [H], i =1,2,3,4,5$
  \State Set $\mathcal{U}_{1,h}$ to be the set containing all the $\tilde{\theta}_h$ that makes $\tilde{P}_h$ well-defined, $h \in [H]$.
    \For { episode $k = 1,2,\cdots, K$}
        \State Calculate $\pi_k, \tilde{\theta}_k, R_k = \argmax_{\pi,\tilde{\theta}_h \in \mathcal{U}_{k,h}, R}{V}^{\pi,\tilde{P}}_{k,1}(s_1, R)$, where ${V}$ is defined in Eqn~\ref{eqn: definition of tildeV}.
        \For{step $h = 1,2,\cdots, H$}
            \State Take action according to the policy $\pi_{k,h}$ and observe $s_{k,h+1} \sim P_h(\cdot|s_{k,h},a_{k,h})$
        \EndFor
        \For{step $h = 1,2,\cdots, H$}
            \State Update $\{\Lambda_{i,k+1,h}\}_{i=1}^5$ using Eqn~\ref{eqn: definition of Lambda1}, \ref{eqn: definition of Lambda2}, \ref{eqn: definition of Lambda3}, \ref{eqn: definition of Lambda4} and \ref{eqn: definition of Lambda5}
            \State Update the model estimation $\{\hat{\theta}_{i,k+1,h}\}_{i=1}^{5}$ using Eqn~\ref{eqn: definition of theta1}, \ref{eqn: definition of theta2}, \ref{eqn: definition of theta3}, \ref{eqn: definition of theta4} and \ref{eqn: definition of theta5}
            \State Add the constraints (Eqn~\ref{inq: confidence constraint 1,2,3,4}) to the confidence set $\mathcal{U}_{k,h}$, and obtain $\mathcal{U}_{k+1,h}$
        \EndFor
    \EndFor
    \State Output: $\left\{\tilde{P}_{K,h}(\cdot|\cdot,\cdot) = \tilde{\theta}_{K,h}^{\top} \phi (\cdot,\cdot,\cdot)\right\}_{h=1}^{H}$.
  \end{algorithmic}
\end{algorithm}

 Our algorithm is described in Algorithm~\ref{alg: exploration phase, Bernstein}. At a high level, Algorithm~\ref{alg: exploration phase, Bernstein} maintains a high-confidence set  $\mathcal{U}_{k,h}$ for the real parameter $\theta_h$ in each episode $k$, and calculates an optimistic value function $V_{k,1}^{\pi,\tilde{P}}(s_1,R)$ for any reward function $R$ and transition $\tilde{P}_h = \tilde{\theta}^{\top}_h \phi$ with $\tilde{\theta}_h \in \mathcal{U}_{k,h}$. Roughly speaking, the value function $V_{k,1}^{\pi,\tilde{P}}(s_1,R)$ measures the expected uncertainty along the whole trajectory induced by policy $\pi$ in the MDP with transition $\tilde{P}$ and reward $R$. To collect more ``informative'' samples and minimize the worst-case uncertainty over all possible transition dynamics and reward function, we calculate $\pi_k = \argmax_{\pi} \max_{\tilde{\theta}_h \in \mathcal{U}_{k,h}, R} V_{k,1}^{\pi,\tilde{P}}(s_1,R)$, and execute the policy $\pi_k$ to collect more data in episode $k$. To ensure that the model estimation $\tilde{\theta}_{k+1,h}$ is close to the true model ${\theta}_{h}$ w.r.t. features $\phi_{V}(s,a)$ of different value functions $V$, we use the samples collected so far to calculate five model estimation $\{\hat{\theta}_{i,k,h}\}_{i=1}^{5}$ and the corresponding constraints at the end of the episode $k$. Each constraint is an ellipsoid in the parameter space centered at the parameter estimation $\hat{\theta}_{i,k,h}$ with covariance matrix $\Lambda_{i,k,h}$ and radius $\beta_{i}$, i.e.
 \begin{align}
    \label{inq: confidence constraint 1,2,3,4}
     \|\tilde{\theta}_h - \hat{\theta}_{i,k,h}\|_{\Lambda_{i,k,h}} \leq \beta_i.
 \end{align}

 We update $\mathcal{U}_{k+1,h}$ by adding these constraints to the the confidence set $\mathcal{U}_{k,h}$.

The remaining problems are how to define the value function $V_{k,1}^{\pi,\tilde{P}}(s_1,R)$ that represents the expected uncertainty for policy $\pi$, transition $\tilde{P}$ and reward $R$, and how to update the model estimation $\theta_{i,k,h}$ and the confidence set $\mathcal{U}_{k,h}$.

\paragraph{Uncertainty Measure} Instead of using $\max_{V}\|\phi_V(s,a)\|_{\Lambda^{-1}_{k,h}}$ to measure the maximum uncertainty for the state-action pair $(s,a)$ in Algorithm~\ref{alg: exploration phase}, we separately define the uncertainty along the trajectories induced by different policy $\pi$, reward function $R$ and transition dynamics $\tilde{P}_h = \tilde{\theta}_h^{\top} \phi$. Specifically, Recall that $\hat{V}_{h}^{\pi,\tilde{P}}(s,R)$ is the value function of policy $\pi$ in the MDP model with transition $\tilde{P}$ and reward $R$. We define the following exploration-driven reward for transition $\tilde{P}$, reward $R$ and policy $\pi$:

\begin{align}
    \label{eqn: exploration-driven reward, alg2}
    u^{\pi,\tilde{P}}_{1,k,h}(s,a,R) = \hat{\beta} \left\|\sum_{s'}\phi(s,a,s')\hat{V}_{h+1}^{\pi,\tilde{P}}(s',R)\right\|_{\left(\Lambda_{1,k,h}\right)^{-1}}.
\end{align}

Suppose $\tilde{V}^{\pi,\tilde{P}}_{k,H+1}(s,R) = 0, \forall s \in \mathcal{S}$, we define the corresponding value function recursively from step $H+1$ to step $1$. That is,
\begin{align}
\label{eqn:value function with true transition and exploration reward}
    \tilde{V}^{\pi,\tilde{P}}_{k,h}(s,R) = \min&\left\{ u^{\pi,\tilde{P}}_{1,k,h}(s,\pi_h(s),R) + {P}_h\tilde{V}^{\pi,\tilde{P}}_{k,h+1} (s,\pi_h(s),R), H\right\}, \forall s \in \mathcal{S}.
\end{align}

$\tilde{V}^{\pi,\tilde{P}}_{k,h}(s,R)$ can be regarded as the expected uncertainty along the trajectories induced by policy $\pi$ for the value function $\hat{V}_{h}^{\pi,\tilde{P}}(s,R)$. In each episode $k$, we wish to collect samples by executing the policy $\pi$ that maximizes the uncertainty measure $\max_{\tilde{\theta}\in \mathcal{U}_{k,h}, R}\tilde{V}^{\pi,\tilde{P}}_{k,1}(s,R)$. However, the definition of $\tilde{V}^{\pi,\tilde{P}}_{k,1}(s,R)$ explicitly depends on the real transition dynamics $P$, which is unknown to the agent. To solve this problem, we construct an optimistic estimation of $\tilde{V}^{\pi,\tilde{P}}_{k,h}(s,R)$. We define $V_{k,h}^{\pi,\tilde{P}}(s, R)$ from step $H+1$ to step $1$ recursively. Suppose $V_{k,H+1}^{\pi,\tilde{P}}(s, R) = 0$. We calculate 
\begin{align}
    \label{eqn: confidence bonus, alg2}
    &u_{2,k,h}^{\pi,\tilde{P}}(s,a,R)  = \hat{\beta} \left\|\sum_{s'}\phi(s,a,s'){V}_{k,h+1}^{\pi,\tilde{P}}(s',R)\right\|_{\left(\Lambda_{2,k,h}\right)^{-1}}, \\
    \label{eqn: definition of tildeV}
    {V}^{\pi,\tilde{P}}_{k,h}(s, R) &= \min\left\{ u_{1,k,h}^{\pi,\tilde{P}}(s,\pi_h(s), R) + u_{2,k,h}^{\pi,\tilde{P}}(s,\pi_h(s), R)+ \tilde{P}_h {V}^{\pi,\tilde{P}}_{k,h+1}(s,\pi_h(s),R),H\right\},
\end{align}
where $u^{\pi,\tilde{P}}_{2,k,h}$ is the confidence bonus which ensures that the optimism $V_{k,h}^{\pi,\tilde{P}}(s, R) \geq \tilde{V}^{\pi,\tilde{P}}_{k,h}(s,R)$ holds with high probability. 
After calculating $V_{k,h}^{\pi,\tilde{P}}(s, R)$, we take maximization over all $\tilde{\theta}_h \in \mathcal{U}_{k,h}$ and reward function $R$, and calculate $\pi_k = \argmax_{\pi} \max_{\tilde{\theta}_h \in \mathcal{U}_{k,h}, R} V_{k,1}^{\pi,\tilde{P}}(s_1,R)$. We execute $\pi_k$ to collect more samples in episode $k$.

\paragraph{Weighted Ridge Regression and Confidence Set} In Algorithm~\ref{alg: exploration phase, Bernstein}, we maintain five model estimation $\{\hat{\theta}_{i,k,h}\}_{i=1}^5$ and add five constraints to the confidence set $\mathcal{U}_{k,h}$ in each episode. These constraints can be roughly classified into three categories. The first and the third constraints are applied to ensure that $\tilde{P}_{k,h}\hat{V}_{h+1}^{\pi,\tilde{P}}(s,a,R)$ is an accurate estimation of ${P}_{k,h}\hat{V}_{h+1}^{\pi,\tilde{P}}(s,a,R)$ for any reward function $R$ and $\tilde{P}_{k,h} = \tilde{\theta}_{k,h}^{\top} \phi$ satisfying $\tilde{\theta}_{k,h} \in \mathcal{U}_{k,h}$, while the second and the forth constraints are used to guarantee that $\tilde{P}_{k,h}{V}_{h+1}^{\pi,\tilde{P}}(s,a,R)$ is an accurate estimation of ${P}_{h}{V}_{h+1}^{\pi,\tilde{P}}(s,a,R)$ for any reward function $R$ and $\tilde{P}_{k,h}= \tilde{\theta}_{k,h}^{\top} \phi$ satisfying $\tilde{\theta}_{k,h} \in \mathcal{U}_{k,h}$. The last constraint is applied due to technical issue and will be explained in Appendix~\ref{appendix: berstein algorithm details}. In this subsection, we introduce the basic idea behind the construction of the first and the third constraints. The second and the forth constraints follow the same idea but consider the different value function ${V}_{h+1}^{\pi,\tilde{P}}(s,a,R)$. 
The formal definition of the parameters $\{\hat{\theta}_{i,k,h}\}_{i=1}^5$ and the constraints are deferred to Appendix~\ref{appendix: berstein algorithm details}. 
 
For notation convenience, we use $V_{k,h}(s)$ as a shorthand of $\hat{V}_{k,h}^{\pi_k,\tilde{P}_k}(s,R_k)$ in this part. The construction of our confidence sets is inspired by a recent algorithm called UCRL-VTR$^+$ proposed by~\cite{zhou2020nearly}. Recall that we use  $(s_{k,h},a_{k,h})$ to denote the state-action pair that the agent encounters at step $h$ in episode $k$.  
We use the following ridge regression estimator to calculate the corresponding $\hat{\theta}_{1,k,h}$ in episode $k$:
\begin{align}
\label{eqn: ridge regression eqn 1}
    \hat{\theta}_{1,k,h} =& \argmin_{\theta \in \mathbb{R}^d} \lambda \|\theta\|_2^2 + \sum_{t=1}^{k-1} \left[\theta^{\top} \phi_{V_{t,h+1}/\bar{\sigma}_{1,t,h}}(s_{t,h},a_{t,h}) - V_{t,h+1}(s_{t,h+1},R_t)/\bar{\sigma}_{1,t,h}\right]^2,
\end{align}
where $\bar{\sigma}^2_{1,k,h} = {\max\left\{H^2/d, [\bar{\mathbb{V}}_{k,h}V_{k,h+1}](s_{k,h},a_{k,h}) + E_{1,k,h}\right\}}$ is an optimistic estimation of the one-step transition variance: $$\mathbb{V}_h V_{k,h+1}(s_{k,h},a_{k,h}) = \mathbb{E}_{s'\sim P_h(\cdot|s_{k,h},a_{k,h})}\left[\left(V_{k,h+1}(s') - P_hV_{k,h}(s_{k,h},a_{k,h})\right)^2\right].$$ 

In the definition of $\bar{\sigma}^2_{1,k,h}$,  $\bar{\mathbb{V}}_{k,h}V_{k,h+1}(s_{k,h},a_{k,h})$ is an empirical estimation for the variance $\mathbb{V}_h V_{k,h+1}(s_{k,h},a_{k,h})$, and $E_{1,k,h}$ is a bonus term defined in Eqn.~\ref{eqn: definition of E} which ensures that $\bar{\sigma}_{k,h}$ is an optimistic estimation of $\mathbb{V}_h V_{k,h+1}(s_{k,h},a_{k,h})$. One technical issue here is how to estimate the variance $\mathbb{V}_h V_{k,h+1}(s_{k,h},a_{k,h})$. Recall that by definition,
\begin{align}
    \mathbb{V}_hV_{k,h+1}(s_{k,h},a_{k,h}) =& P_hV_{k,h+1}^2(s_{k,h},a_{k,h}) - [P_hV_{k,h+1}(s_{k,h},a_{k,h})]^2 \\
    = &\theta^{\top}_{h} \phi_{V^2_{k,h+1}}(s_{k,h},a_{k,h}) - [\theta^{\top}_{h} \phi_{V_{k,h+1}}(s_{k,h},a_{k,h})]^2.
\end{align}
We use $\tilde{\theta}_{k,h}^{\top} \phi_{V_{k,h+1}^2}(s_{k,h},a_{k,h}) -\left[\tilde{\theta}^{\top}_{k,h} \phi_{V_{k,h+1}}(s_{k,h},a_{k,h})\right]^2$ as our variance estimator $\bar{\mathbb{V}}_{k,h}V_{k,h+1}(s_{k,h},a_{k,h})$, where $\tilde{\theta}_{k,h} \in \mathcal{U}_{k,h}$ is the parameter which maximizes the value $V_{k, 1}^{\pi, \tilde{P}}\left(s_{1}, R\right)$ in episode $k$.
To ensure that $\bar{\mathbb{V}}_{k,h}V_{k,h+1}(s_{k,h},a_{k,h})$ is an accurate estimator, we maintain another parameter estimation $\hat{\theta}_{3,k,h}$ using history samples w.r.t. the feature $\phi_{V_{k,h+1}^2}(s_{k,h},a_{k,h})$. 
\begin{align}
    \hat{\theta}_{3,k,h} = \argmin_{\theta \in \mathbb{R}^d} \lambda \|\theta\|_2^2 + \sum_{t=1}^{k-1} \left[\theta^{\top} \phi_{V^2_{t,h+1}}(s_{t,h},a_{t,h}) - V^2_{t,h+1}(s_{t,h+1})\right]^2.
\end{align}

After calculating $\hat{\theta}_{1,k,h}$ and $\hat{\theta}_{3,k,h}$, we add the first and the third constraints (Eqn.~\ref{inq: confidence constraint 1,2,3,4}) to the confidence set $\mathcal{U}_{k,h}$, where $\Lambda_{1,k,h}$ and $\Lambda_{3,k,h}$ is the corresponding covariance matrix of all history samples, i.e. $\Lambda_{1,k,h} = \sum_{t=1}^{k-1} \bar{\sigma}_{1,t,h}^{-2} \phi_{V_{t,h+1}}(s_{t,h},a_{t,h})\phi^{\top}_{V_{t,h+1}}(s_{t,h},a_{t,h})$ and $\Lambda_{3,k,h} = \sum_{t=1}^{k-1}  \phi_{V^2_{t,h+1}}(s_{t,h},a_{t,h})\phi^{\top}_{V^2_{t,h+1}}(s_{t,h},a_{t,h})$. 



\subsection{Regret}

 We present the regret upper bound of Algorithm~\ref{alg: exploration phase, Bernstein} in Theorem~\ref{theorem: new main}. In the regime where $d \geq H$ and $\epsilon \leq H/\sqrt{d}$, we can obtain $\tilde{O}(d^2H^3/\epsilon^2)$ sample complexity upper bound, which matches the sample complexity lower bound except logarithmic factors. 
\begin{theorem}
\label{theorem: new main}
With probability at least $1-\delta$, after collecting $K=\tilde{O}\left(\frac{d^2H^3+dH^4}{\epsilon^2} + \frac{d^{2.5}H^2+ d^2H^3}{\epsilon}\right)$ trajectories, Algorithm~\ref{alg: exploration phase, Bernstein} returns a transition model $\tilde{P}_K$, then for any given reward in the planning phase, a policy returned by any $\epsilon_{\rm opt}$-optimal plug-in solver on $(S, A, \tilde{P}_K, R,H,\nu)$ is $O(\epsilon + \epsilon_{\rm opt})$-optimal for the true MDP, $(S, A, P, R,H,\nu)$.

\end{theorem}

\section{Conclusion}
\label{sec: conclusion}
This paper studies the sample complexity of plug-in solver approach for reward-free reinforcement learning. We propose a statistically efficient algorithm with sample complexity $\tilde{O}\left(d^2H^4/\epsilon^2\right)$. We further refine the complexity by providing an another algorithm with sample complexity $\tilde{O}\left(d^2H^3/\epsilon^2\right)$ in certain parameter regimes. To the best of our knowledge, this is the first minimax sample complexity bound for reward-free exploration with linear function approximation. 
As a side note, our approaches provide an efficient learning method for the RL model representation, which preserves values and policies for any other down-stream tasks (specified by different rewards). 

Our sample complexity bound matches the lower bound only when $d \geq H$ and $\epsilon \leq H/\sqrt{d}$. It is unclear whether minimax rate could be obtained in a more broader parameter regimes. We plan to address this issue in the future work.

\section{Acknowledgments}
Liwei Wang was supported by National Key R\&D Program of China (2018YFB1402600), Exploratory Research Project of Zhejiang Lab (No. 2022RC0AN02), BJNSF (L172037), Project 2020BD006 supported by
PKUBaidu Fund.

\bibliography{references}
\bibliographystyle{iclr2022_conference}

\appendix
\newpage
\section{Omitted Details in Section~\ref{sec: reward-free algorithm}}
\label{Appendix: Hoeffding algorithm}

\subsection{Notations}
In this subsection, we summarize the notations used in Section~\ref{Appendix: Hoeffding algorithm}.

 \begin{tabular}{ll}

\hline
\textbf{Symbol} & \textbf{Explanation}\\
\hline
$\mathcal{E}_1$ & The high-probability event for Theorem~\ref{theorem: main} \\
$s_{k,h},a_{k,h}$ & The state and action that the agent encounters in episode $k$ and step $h$ \\
$\tilde{V}_{k,h+1,s,a}$ & The value function with maximum uncertainty: $\argmax_{V\in \mathcal{V}}\left\|\sum_{s'} {\phi}(s,a,s')V(s')\right\|_{({ {\Lambda}}_{k,h})^{-1}}$ \\
$\phi_{k,h}(s,a)$ & $\sum_{s'} {\phi}(s,a,s')\tilde{V}_{k,h+1,s,a}(s')$ \\
$\hat{\theta}_{k,h}$ & The estimation of $\theta_{h}$ in episode $k$: $\left( {\Lambda}_{k,h}\right)^{-1} \sum_{t=1}^{k-1}  {\phi}_{t,h}(s_{t,h},a_{t,h}) \tilde{V}_{t,h+1,s,a}(s_{t,h+1})$ \\
${\Lambda}_{k,h}$ & The covariance matrix in $(k,h)$: $\sum_{t=1}^{k-1}  {\phi}_{t,h}(s_{t,h},a_{t,h}) {\phi}_{t,h}(s_{t,h},a_{t,h})^{\top} + \lambda I$\\
$u_{k,h}(s,a)$ & The uncertainty measure: $\beta  \sqrt{ {\phi}_{k,h}(s, a)^{\top}\left( {\Lambda}_{k,h}\right)^{-1}  {\phi}_{k,h}(s, a)}$ \\
$R_{k,h}(s,a)$ & The exploration-driven reward which equals $u_{k,h}(s,a)$ \\
$Q_{k,h}(s,a)$ & The Q function defined in line 12 of Algorithm~\ref{alg: exploration phase} \\
$V_{k,h}(s)$ & The value function defined in line 13 of Algorithm~\ref{alg: exploration phase} \\
$\tilde{V}^*_{h}(s,R)$ & The value function defined in Eqn~\ref{eqn: tildeV for Alg 1}\\
$\tilde{Q}^*_h(s,R)$  & The Q value similarly defined as $\tilde{V}^*_{h}(s,R)$ \\
$\tilde{V}^{\pi}_{h}(s,R)$ & The value function of policy $\pi$ similarly defined as $\tilde{V}^*_{h}(s,R)$\\
$\tilde{\theta}_h$ & The parameter estimation returned at the end of the exploration phase \\
$\hat{P}_{k,h}(s'|s,a)$ & $\hat{\theta}_{k,h}^{\top} \phi(s,a,s')$ \\
$\tilde{P}_h(s'|s,a)$ & $\tilde{\theta}_{h}^{\top} \phi(s,a,s')$ \\
\hline
\end{tabular}   

\subsection{Proof Overview}
Now we briefly explain the main idea in the proof.  Firstly, we introduce the value function $\tilde{V}^*_h(s,R)$, which is recursively defined from step $H+1$ to step $1$: 
\begin{align}
    \label{eqn: tildeV for Alg 1}
    &\tilde{V}^*_{H+1}(s,R) = 0, \forall s \in \mathcal{S} \\
    \tilde{V}^*_{h}(s,R) = & \max_{a \in \mathcal{A}} \left\{\min\left\{R_h(s,a) + P_h\tilde{V}^*_{h+1}(s,a,R),H\right\}\right\}, \forall s \in \mathcal{S}, h \in [H]
\end{align}
Compared with the definition of $V_h^*(s,R)$, the main difference is that we take minimization over the value and $H$ at each step. We state the following lemma, which gives an upper bound on the sub-optimality gap of $\hat{\pi}$ in the planning phase.
\begin{lemma}
\label{lemma: sub-optimality gap, main page}
With probability at least $1-\delta$, the sub-optimality gap of the policy $\hat{\pi}_R$ for any reward function $R$ in the planning phase satisfies $V_1^*(s_1,R) - V_1^{\hat{\pi}_R}(s_1,R)   \leq 4  \tilde{V}_1^*(s_1, R_K) + \epsilon_{\rm opt}$, where $R_K$ is the exploration-driven reward used for episode $K$ in the exploration phase.
\end{lemma}

This lemma connects the sub-optimality gap with the value function of the auxiliary reward in episode $K$. So the remaining problem is how to upper bound $\tilde{V}_1^*(s_1, R_K)$. Since the exploration-driven reward $R_k$ is non-increasing w.r.t. $k$, it is not hard to prove that $K \tilde{V}_1^*(s_1, R_K) \leq \sum_{k=1}^{K} \tilde{V}_1^{*}(s_1,R_k)$. We use the following two lemmas to upper bound $\sum_{k=1}^{K} \tilde{V}_1^{*}(s_1,R_k)$.

\begin{lemma}
\label{lemma: optimism in exploration phase, main page}
With probability at least $1-\delta$, $\tilde{V}_h^*(s,R_k) \leq V_{k,h}(s)$ holds for any $(s,a) \in \mathcal{S}\times \mathcal{A}, h \in [H]$ and $k \in [K]$. 
\end{lemma}

\begin{lemma}
\label{lemma: upper bound of sum V_k, main page}
With probability at least $1-\delta$, $$ \sum_{k=1}^{K} V_{k,1}(s_{1}) \leq 6 H^2d\sqrt{K \log(4H^3KB^2/\delta) \log(1+KH^2B^2/d)}.$$
\end{lemma}

With the help of the optimistic bonus term $u_{k,h}$, we can prove that the estimation value $V_{k,h}$ is always optimistic w.r.t $\tilde{V}^*_h$, which is illustrated in Lemma~\ref{lemma: optimism in exploration phase, main page}. Therefore, we have $\sum_{k=1}^{K} \tilde{V}_1^{*}(s_1,R_k) \leq \sum_{k=1}^{K} V_{k,1}(s_{1})$. By adapting the regret analysis in the standard RL setting to the reward-free setting, we can upper bound the summation of $V_{k,1}(s_{1})$ in Lemma~\ref{lemma: upper bound of sum V_k, main page}. Combining the above two lemmas, we derive the upper bound of $\tilde{V}^{*}_1(s_1,R_K)$, then we bound the sub-optimality gap of $\hat{\pi}$ by Lemma~\ref{lemma: sub-optimality gap, main page}.

\subsection{High-probability Events}
We firstly state the following high-probability events.
\begin{lemma}
\label{lemma: confidence set for theta}
With probability at least $1-\delta/2$, the following inequality holds for any $k \in [K],h \in [H]$:
\begin{align}
    \|\hat{ {\theta}}_{k,h} -  {\theta}_h \|_{ {\Lambda}_{k,h}} \leq \beta.
\end{align}
\end{lemma}

\begin{proof}
For some step $h \in [H]$, the agent selects an action with feature $x_{k,h} = \phi_{k,h}(s_{k,h},a_{k,h})$. The noise $\eta_{k,h} = \tilde{V}_{k,h,s_{k,h},a_{k,h}}(s') - \theta_{h}^{\top}\phi_{k,h}(s_{k,h},a_{k,h})$ satisfies $H$-sub-Gaussian. By Lemma~\ref{lemma: self-normalized bound}, the following inequality holds with probability $1-\frac{\delta}{2H}$:
\begin{align}
     \|\hat{ {\theta}}_{k,h} -  {\theta}_h \|_{ {\Lambda}_{k,h}} \leq H\sqrt{d \log(\frac{4H^3K}{\lambda\delta})} + \lambda^{1/2}B.
\end{align}
By taking union bound over all $h \in [H]$, we can prove the lemma.

\end{proof}

\begin{lemma}
\label{lemma: martingale difference}
With probability at least $1-\delta/2$, we have 
\begin{align}
    \sum_{k=1}^{K}\sum_{h=1}^{H}\left(P_h V_{k,h+1}(s_{k,h},a_{k,h}) - V_{k,h+1}(s_{k,h+1}) \right) \leq \sqrt{2H^3K \log(4/\delta)}.
\end{align}
\end{lemma}
\begin{proof}
 This lemma follows directly by Azuma's inequality.
\end{proof}

During the following analysis, we denote the high-probability events defined in Lemma~\ref{lemma: confidence set for theta} and Lemma~\ref{lemma: martingale difference} as $\mathcal{E}_1$.

\subsection{Proof of Lemma~\ref{lemma: sub-optimality gap, main page}}
\begin{lemma}
\label{lemma: sub-optimality gap} 
(Restatement of Lemma~\ref{lemma: sub-optimality gap, main page})
Under event $\mathcal{E}_1$, the sub-optimality gap of the policy $\hat{\pi}_R$ for any reward function $R$ in the planning phase can be bounded by
\begin{align}
    V_1^*(s_1,R) - V_1^{\hat{\pi}_R}(s_1,R)  & \leq 4 \tilde{V}_1^*(s_1, R_K) + \epsilon_{\rm opt}.
\end{align}
\end{lemma}

\begin{proof}
\begin{align}
    \label{inq: sub-optimality gap, eqn1}
     &V_1^*(s_1,R) - V_1^{\hat{\pi}_R}(s_1,R)  \\
     = & \left(V_1^*(s_1,R) -\hat{V}_1^{\pi^*_R,\tilde{P}}(s_1,R)\right)  + \left(\hat{V}_1^{\hat{\pi}_R,\tilde{P}}(s_1,R) -V_1^{\hat{\pi}_R}(s_1, R)\right) + \left(\hat{V}_1^{\pi^*_R,\tilde{P}}(s_1,R)- \hat{V}_1^{\hat{\pi}_R,\tilde{P}}(s_1,R)\right) \\
     \leq & \left(V_1^*(s_1,R) -\hat{V}_1^{\pi^*_R,\tilde{P}}(s_1,R)\right)+ \left(\hat{V}_1^{\hat{\pi}_R,\tilde{P}}(s_1,R) -V_1^{\hat{\pi}_R}(s_1, R)\right) + \epsilon_{\rm opt}.
\end{align}
The inequality is because that the policy $\hat{\pi}_R$ is the $\epsilon_{\rm opt}$-optimal policy in the estimated MDP $\hat{M}$, i.e. $\hat{V}_1^{\pi^*_R,\tilde{P}}(s_1,R) \leq \hat{V}_1^{\hat{\pi}_R,\tilde{P}}(s_1,R)+ \epsilon_{\rm opt}$ .

For the notation convenience, for a certain sequence $\{R_h\}_{h=1}^{H}$, we define the function $W_{h}(\{R_h\})$ recursively from step $H+1$ to step $1$. Firstly, we define $W_{H+1}(\{R_h\}) = 0$. $W_h(\{R_h\})$ is calculated recursively from $W_{h+1}(\{R_h\})$:
\begin{align}
    W_h(\{R_h\}) = \min\left\{H, R_h+  W_{h+1}(\{R_h\})\right\}.
\end{align}

Similarly with the definition of $\tilde{V}^*_h(s,R)$, we introduce the value function $\tilde{V}^{\pi}_h(s,R)$, which is recursively defined from step $H+1$ to step $1$: 
\begin{align}
    \label{eqn: tildeVpi for Alg 1}
    &\tilde{V}^{\pi}_{H+1}(s,R) = 0, \forall s \in \mathcal{S} \\
    \tilde{V}^{\pi}_{h}(s,R) = &  \left\{\min\left\{R_h(s,\pi(s)) + P_h\tilde{V}^{\pi}_{h+1}(s,\pi(s),R),H\right\}\right\}, \forall s \in \mathcal{S}, h \in [H].
\end{align}

We use $\operatorname{traj} \sim (\pi,P)$ to indicate that the trajectory $\{s_h,a_h\}_{h=1}^{H}$ is sampled from transition $P$ with policy $\pi$. For any policy $\pi$, we have 
\begin{align}
    & \left|\mathbb{E}_{s_1 \sim \mu}\left(\hat{V}^{\pi, \tilde{P}}_1(s_1, R) - V^{\pi}_1(s_1)\right)\right| \\
    = & \left|\mathbb{E}_{\operatorname{traj} \sim (\pi,P)} W_1\left(\left\{(\tilde{P}_{h} - P_h)\hat{V}_{h+1}^{\pi, \tilde{P}_{K}}(s_h,a_h, R)\right\}\right)\right|  \\
    = & \left|\mathbb{E}_{\operatorname{traj} \sim (\pi,P)} W_1\left(\left\{(\tilde{\theta}_{h} - \theta_h)\sum_{s'}\phi(s_h,a_h,s')\hat{V}_{h+1}^{\pi, \tilde{P}}(s', R)\right\}\right)\right|\\
    \leq & \mathbb{E}_{\operatorname{traj} \sim (\pi,P)} W_1\left(\left\{\left\|\tilde{\theta}_{h} - \theta_h\right\|_{\Lambda_{K,h}} \left\|\sum_{s'}\phi(s_h,a_h,s')\hat{V}_{h+1}^{\pi, \tilde{P}}(s', R)\right\|_{\left(\Lambda_{K,h}\right)^{-1}}\right\}\right) \\
    \leq & \mathbb{E}_{\operatorname{traj} \sim (\pi,P)} W_1\left(\left\{2\beta \left\|\sum_{s'}\phi(s_h,a_h,s')\hat{V}_{h+1}^{\pi, \tilde{P}}(s', R)\right\|_{\left(\Lambda_{K,h}\right)^{-1}} \right\}\right)  \\
    \leq &  2\mathbb{E}_{\operatorname{traj} \sim (\pi,P)} W_1\left(\left\{ u_{K,h}(s_h,a_h) \right\}\right) \\
    = & 2\tilde{V}^{\pi}_{1}(s_1,R_K) \\
    \leq & 2\tilde{V}^{*}_{1}(s_1,R_K).
\end{align}
The second inequality is due to lemma~\ref{lemma: confidence set for theta} and the definition of $\tilde{\theta}$. Plugging this inequality back to Inq.~\ref{inq: sub-optimality gap, eqn1}, we can prove the lemma.
\end{proof}

\subsection{Proof of Lemma~\ref{lemma: optimism in exploration phase, main page}}

\begin{lemma}
\label{lemma: optimism in exploration phase}
(Restatement of Lemma~\ref{lemma: optimism in exploration phase, main page}) Under event $\mathcal{E}_1$, $\tilde{V}_h^*(s,R_k) \leq V_{k,h}(s)$ holds for any $(s,a) \in \mathcal{S}\times \mathcal{A}, h \in [H]$ and $k \in [K]$. . 
\end{lemma}
\begin{proof}
We prove the lemma by induction. Suppose $\tilde{V}_{h+1}^*(s,R_k) \leq V_{k,h+1}(s)$,
\begin{align}
    &\tilde{Q}_h^*(s,a,R_k) - Q_{k,h}(s,a) \\
    \leq& -u_{k,h}(s,a) + \left( \theta_{k,h}-\hat{\theta}_{k,h}\right)^{\top} \cdot \sum_{s'} \phi(s,a,s') V_{k,h}(s') + P_h(\tilde{V}_{h+1}^* - V_{k,h+1})(s,a) \\
    \leq & -u_{k,h}(s,a) + \left( \theta_{k,h}-\hat{\theta}_{k,h}\right)^{\top} \cdot \sum_{s'} \phi(s,a,s') V_{k,h}(s') \\
    \leq & -u_{k,h}(s,a) + \left\| \theta_{k,h}-\hat{\theta}_{k,h}\right\|_{ {\Lambda}_{k,h}} \left\|\sum_{s'} \phi(s,a,s') V_{k,h}(s')\right\|_{( {\Lambda}_{k,h})^{-1}} \\
    \leq & -u_{k,h}(s,a) + \left\| \theta_{k,h}-\hat{\theta}_{k,h}\right\|_{ {\Lambda}_{k,h}} \max_{V \in \mathcal{V}}\left\|\sum_{s'} \phi(s,a,s') V(s')\right\|_{( {\Lambda}_{k,h})^{-1}} \\
    \leq & -u_{k,h}(s,a) + \beta \max_{V \in \mathcal{V}}\left\|\sum_{s'} \phi(s,a,s') V(s')\right\|_{( {\Lambda}_{k,h})^{-1}} \\
    = & 0.
\end{align}
The first inequality is due to induction condition $\tilde{V}_{h+1}^*(s,R_k) \leq V_{k,h+1}(s)$. The last inequality is due to Lemma~\ref{lemma: confidence set for theta}.
Since $\tilde{Q}_h^*(s,a) \leq Q_{k,h}(s,a,R_k)$ for any $a \in \mathcal{A}$, we have $\tilde{V}_h^*(s,R_k) \leq V_{k,h}(s)$.
\end{proof}

\subsection{Proof of Lemma~\ref{lemma: upper bound of sum V_k, main page}}
\begin{lemma}
\label{lemma: upper bound of sum V_k}
(Restatement of Lemma~\ref{lemma: upper bound of sum V_k, main page}) Under event $\mathcal{E}_1$, $\sum_{k=1}^{K} V_{k,1}(s_{k,1}) \leq 6 H^2d\sqrt{K \log(4H^3KB^2/\delta) \log(1+KH^2B^2/d)}$.
\end{lemma}
\begin{proof}
\begin{align}
\label{eqn: upper bound of sum V_k, eqn1}
    &Q_{k,h}(s_{k,h},a_{k,h}) \\
    =& \min\left\{H, R_{k,h}(s_{k,a},a_{k,h}) + u_{k,h}(s_{k,a},a_{k,h}) + \hat{P}_{k,h} V_{k,h+1}(s_{k,h},a_{k,h})\right\}  \\
    \leq &  \min\left\{H, 2\beta\left\|\phi_{k,h}(s_{k,h},a_{k,h})\right\|_{\Lambda_{k,h}^{-1}} + \hat{P}_{k,h} V_{k,h+1}(s_{k,h},a_{k,h})\right\}\\
    \leq &   \min\left\{H, 2 \beta\left\|\phi_{k,h}(s_{k,h},a_{k,h})\right\|_{\Lambda_{k,h}^{-1}}\right\} + \min\left\{H,\left(\hat{P}_{k,h}- P_h \right)V_{k,h+1}(s_{k,h},a_{k,h})\right\} \\
    & + \left(P_h V_{k,h+1}(s_{k,h},a_{k,h}) - V_{k,h+1}(s_{k,h+1})\right) + V_{k,h+1}(s_{k,h+1}) \\
    \leq &3 \beta \min\left\{1, \left\|\phi_{k,h}(s_{k,h},a_{k,h})\right\|_{\Lambda_{k,h}^{-1}}\right\} 
    + \left(P_h V_{k,h+1}(s_{k,h},a_{k,h}) - V_{k,h+1}(s_{k,h+1})\right) + V_{k,h+1}(s_{k,h+1}).
\end{align}
The last inequality is from
\begin{align}
    \left(\hat{P}_{k,h} - P_h \right)V_{k,h+1}(s_{k,h},a_{k,h}) 
    = & (\hat{\theta}_{k,h} - \theta_h) \phi_{V_{k,h+1}}(s_{k,h},a_{k,h}) \\
    \leq & \left\|\hat{\theta}_{k,h} - \theta_h\right\|_{\Lambda_{k,h}} \left\|\phi_{V_{k,h+1}}(s_{k,h},a_{k,h})\right\|_{\Lambda_{k,h}^{-1}} \\
    \leq & \beta  \left\|\phi_{V_{k,h+1}}(s_{k,h},a_{k,h})\right\|_{\Lambda_{k,h}^{-1}} \\
    \leq & \beta  \left\|\phi_{k,h}(s_{k,h},a_{k,h})\right\|_{\Lambda_{k,h}^{-1}}.
\end{align}
From Inq~\ref{eqn: upper bound of sum V_k, eqn1}, we have
\begin{align}
    \sum_{k=1}^{K} V_{k,1}(s_1) = & \sum_{k=1}^{K}\sum_{h=1}^{H}  3 \beta \min\left\{1, \left\|\phi_{k,h}(s_{k,h},a_{k,h})\right\|_{\Lambda_{k,h}^{-1}}\right\} \\
    &+ \sum_{k=1}^{K}\sum_{h=1}^{H}\left(P_h V_{k,h+1}(s_{k,h},a_{k,h}) - V_{k,h+1}(s_{k,h+1})\right).
\end{align}
For the first term, by Lemma~\ref{lemma: data accumulation}, we have for any $h \in [H]$,
\begin{align}
    \sum_{k=1}^{K} \min\left\{1,\left\|\phi_{k,h}(s_{k,h},a_{k,h})\right\|_{\Lambda_{k,h}^{-1}}\right\} \leq & \sqrt{K \sum_{k=1}^{K}\min\left\{1, \left\|\phi_{k,h}(s_{k,h},a_{k,h})\right\|^2_{\Lambda_{k,h}^{-1}}\right\}} \\
    \leq & \sqrt{2 dK \log(1+KH^2/(d\lambda))}.
\end{align}
Since $\lambda = B^{-2}$, we have
\begin{align}
    \sum_{k=1}^{K}\sum_{h=1}^{H} 3\beta \min\left\{1, \left\|\phi_{k,h}(s_{k,h},a_{k,h})\right\|_{\Lambda_{k,h}^{-1}} \right\} \leq 3 H^2d\sqrt{2K \log(4H^3KB^2/\delta) \log(1+KH^2B^2/d)}.
\end{align}
For the second term, by Lemma~\ref{lemma: martingale difference},
\begin{align}
    \sum_{k=1}^{K}\sum_{h=1}^{H}\left(P_h V_{k,h+1}(s_{k,h},a_{k,h}) - V_{k,h+1}(s_{k,h+1}) \right) \leq \sqrt{2H^3K \log(4/\delta)}.
\end{align}
Therefore, we have
\begin{align}
    \sum_{k=1}^{K} V_{k,1}(s_{k,1}) \leq  6 H^2d\sqrt{K \log(4H^3KB^2/\delta) \log(1+KH^2B^2/d)}.
\end{align}

\end{proof}
\begin{lemma}
\label{lemma: upper bound of V star}
 Under event $\mathcal{E}_1$, $ \tilde{V}_1^*(s_1, R_K) \leq  6 H^2d\sqrt{ \frac{\log(4H^3KB^2/\delta) \log(1+KH^2B^2/d)}{K}}$.
\end{lemma}
\begin{proof}
Since $\tilde{V}^*_1(s_{k,1},R_k) \leq V_{k,1}(s_{k,1})$ by lemma~\ref{lemma: optimism in exploration phase}, we have 
\begin{align}
    \sum_{k=1}^{K} \tilde{V}_1^*(s_1, R_k) \leq \sum_{k=1}^{K} V_{k,1}(s_{1}) \leq  6 H^2d\sqrt{K \log(4H^3KB^2/\delta) \log(1+KH^2B^2/d)}.
\end{align}
By the definition of $R_k$, we know that $R_k(s,a)$ is non-increasing w.r.t $k$. Therefore,
\begin{align}
    \sum_{k=1}^{K} \tilde{V}_1^*(s_{1}, R_K) \leq \sum_{k=1}^{K} \tilde{V}_1^*(s_1, R_k) \leq 6 H^2d\sqrt{K \log(4H^3KB^2/\delta) \log(1+KH^2B^2/d)}.
\end{align}

The lemma is proved by dividing both sides by $K$.
\end{proof}

\subsection{Proof of Theorem~\ref{theorem: main}}
 Combining the results in Lemma~\ref{lemma: upper bound of V star} and Lemma~\ref{lemma: sub-optimality gap}, we know that for any reward function $R$, 
 \begin{align}
     V^*_1(s_1,R) - V_1^{\hat{\pi}_R}(s_1,R) \leq 24 H^2d\sqrt{ \frac{\log(4H^3KB^2/\delta) \log(1+KH^2B^2/d)}{K}} + \epsilon_{\rm opt}.
 \end{align}
 
 Choosing $K=\frac{C_1 H^4d^2\log(4H^3KB^2/\delta) \log(1+KH^2B^2/d)}{\epsilon^2}$ for some constant $C_1$ suffices to guarantee that $V_1^*(s_1,R) - V_1^{\hat{\pi}_R}(s_1,R) \leq \epsilon + \epsilon_{\rm opt}$.

\section{Omitted Details in Section~\ref{sec: Berstein algorithm}}
\label{appendix: Omitted Details in Section 5}

\subsection{Notations}
In this subsection, we summarize the notations used in Section~\ref{appendix: Omitted Details in Section 5}.

 \begin{tabular}{ll}

\hline
\textbf{Symbol} & \textbf{Explanation}\\
\hline
$\mathcal{E}_2$ & The high-probability event for Theorem~\ref{theorem: new main} \\
$s_{k,h},a_{k,h}$ & The state and action that the agent encounters in episode $k$ and step $h$ \\
$\phi_{V}(s,a,R)$ & $\sum_{s'} {\phi}(s,a,s')V^{\pi,P}(s',R)$ for certain value function $V$ \\
$\hat{V}_h^{\pi,\tilde{P}}(s,R)$ & The value function of policy $\pi$ in the MDP model with transition $\tilde{P}$ and reward $R$ \\
$ \tilde{V}^{\pi,\tilde{P}}_{k,h}(s,R) $ & The expected uncertainty along the trajectories induced by policy $\pi$ for $\hat{V}_h^{\pi,\tilde{P}}(s,R)$ (Defined in Eqn~\ref{eqn:value function with true transition and exploration reward}) \\
${V}^{\pi,\tilde{P}}_{k,h}(s, R)$ & The optimistic estimation of $ \tilde{V}^{\pi,\tilde{P}}_{k,h}(s,R) $ (Defined in Eqn~\ref{eqn: definition of tildeV}) \\
$u^{\pi,\tilde{P}}_{1,k,h}(s,a,R)$ & The exploration-driven reward for transition $\tilde{P}$, reward $R$ and policy $\pi$ (Defined in Eqn~\ref{eqn: exploration-driven reward, alg2}) \\
$u^{\tilde{P},\pi}_{2,k,h}(s,a,R)$ & The confidence bonus ensuring the optimism $V_{k,h}^{\pi,\tilde{P}}(s, R) \geq \tilde{V}^{\pi,\tilde{P}}_{k,h}(s,R)$ (Defined in Eqn~\ref{eqn: confidence bonus, alg2}) \\
$\mathbb{V}_h V(s,a)$ & The one-step transition variance w.r.t. certain value function $V$ \\
$\bar{\mathbb{V}}_{1,k,h}\left(s, a\right)$ & The empirical variance estimation w.r.t. the value $\hat{V}_h^{\pi_k,\tilde{P}_k}(s,R_k)$ (Defined in Eqn~\ref{eqn: varaince estimation, hatV})\\
$\bar{\mathbb{V}}_{2,k,h}\left(s, a\right)$ & The empirical variance estimation w.r.t. the value ${V}_h^{\pi_k,\tilde{P}_k}(s,R_k)$ (Defined in Eqn~\ref{eqn: varaince estimation, V})\\
$E_{1,k,h}$, $E_{2,k,h}$ & The confidence bonus for the variance estimation $\bar{\mathbb{V}}_{1,k,h}\left(s, a\right)$ and $\bar{\mathbb{V}}_{2,k,h}\left(s, a\right)$, respectively \\
$\bar{\sigma}^2_{1,k,h}$, $\bar{\sigma}^2_{2,k,h}$ & The optimistic variance estimation for $\bar{\mathbb{V}}_{1,k,h}\left(s, a\right)$ and $\bar{\mathbb{V}}_{2,k,h}\left(s, a\right)$, respectively\\
$\hat{\theta}_{i,k,h}$ & The parameter estimation w.r.t. certain value function (Defined in Section~\ref{appendix: berstein algorithm details}) \\
$\Lambda_{i,k,h}$ & The empirical covariance matrix w.r.t. certain value function (Defined in Section~\ref{appendix: berstein algorithm details}) \\
$\mathcal{U}_{k,h}$ & The confidence set containing $\theta_h$ with high probability \\
$\pi_k, \tilde{\theta}_k, R_k$ & $\argmax_{\pi,\tilde{\theta}_h \in \mathcal{U}_{k,h}, R}{V}^{\pi,\tilde{P}}_{k,1}(s_1, R)$ \\
$\tilde{Y}_{k,h}(s)$ & The ``value function'' for the MDP with transition $\tilde{P}_k$ and reward $\bar{\mathbb{V}}_{1,k,h}$ \\
$\tau$ & $\log(32K^2H/\delta)\log^2(1+KH^4B^2)$ \\
\hline
\end{tabular}   

\subsection{Omitted Details of Algorithm~\ref{alg: exploration phase, Bernstein}}
\label{appendix: berstein algorithm details}
In algorithm~\ref{alg: exploration phase, Bernstein}, we maintain five different parameter estimation $\hat{\theta}_{1,k,h},\hat{\theta}_{2,k,h},\hat{\theta}_{3,k,h},\hat{\theta}_{4,k,h}$ and $\hat{\theta}_{5,k,h}$, and the corresponding covariance matrix $\Lambda_{1,k,h},\Lambda_{2,k,h},\Lambda_{3,k,h},\Lambda_{4,k,h}$ and $\Lambda_{5,k,h}$. $\hat{\theta}_{1,k,h},\hat{\theta}_{2,k,h}$ are the parameter estimation using history samples w.r.t the features of variance-normalized value function $\hat{V}_{k,h+1}^{\pi_k,\tilde{P}_{k}}/\bar{\sigma}_{1,k,h}$ and ${V}_{k,h+1}^{\pi_k,\tilde{P}_{k}}/\bar{\sigma}_{2,k,h}$, respectively. $\hat{\theta}_{3,k,h},\hat{\theta}_{4,k,h}$ are the parameter estimation using samples w.r.t the features of value function $\left(\hat{V}_{k,h+1}^{\pi_k,\tilde{P}_{k}}\right)^2$ and $\left({V}_{k,h+1}^{\pi_k,\tilde{P}_{k}}\right)^2$. $\hat{\theta}_{5,k,h}$ is somewhat technical and will be explained later. for any $i \in [5]$ and $h \in [H]$, $\hat{\theta}_{i,1,h}$ are initialized as $\boldsymbol{0}$, and $\Lambda_{i,1,h}$ are initialized as $\lambda I$ in Algorithm~\ref{alg: exploration phase, Bernstein}. $\mathcal{U}_{1,h}$ is the set containing all the $\tilde{\theta}_h$ that makes $\tilde{P}_h$ well-defined, i.e. $\sum_{s'} \tilde{\theta}_h^{\top}\phi(s'|s,a)  = 1$ and $ \tilde{\theta}_h^{\top}\phi(s'|s,a) \geq 0, \forall s,a$.

After observing $\{s_{k,h},a_{k,h}\}_{h=1}^{H}$ in episode $k$, we calculate $\bar{\mathbb{V}}_{1,k,h}\left(s, a\right)$ and $\bar{\mathbb{V}}_{2,k,h}\left(s, a\right)$ as the corresponding variance in the empirical MDP with transition dynamics $\tilde{P}_{k}$. 
\begin{align}
    \label{eqn: varaince estimation, hatV}
    \bar{\mathbb{V}}_{1,k,h}(s,a)  =& \tilde{\theta}_{k,h}^{\top}\sum_{s'} \phi(s,a,s')\left(\hat{V}_{k,h+1}^{\pi_k,\tilde{P}_{k}}(s', R_k)\right)^2
    -\left[\tilde{\theta}_{k,h}^{\top}\sum_{s'} \phi(s,a,s')\hat{V}_{k,h+1}^{\pi_k,\tilde{P}_{k}}(s', R_k)\right]^2, \\
    \label{eqn: varaince estimation, V}
    \bar{\mathbb{V}}_{2,k,h}(s,a) = & \tilde{\theta}_{k,h}^{\top}\sum_{s'} \phi(s,a,s')\left({V}_{k,h+1}^{\pi_k,\tilde{P}_{k}}(s', R_k)\right)^2
    -\left[\tilde{\theta}_{k,h}^{\top}\sum_{s'} \phi(s,a,s'){V}_{k,h+1}^{\pi_k,\tilde{P}_{k}}(s', R_k)\right]^2.
\end{align}
To guarantee the variance estimation is optimistic, we calculate the confidence bonus $E_{1,k,h}$ and $E_{2,k,h}$  for variance estimation $\bar{\mathbb{V}}_{1,k,h}\left(s_{k,h}, a_{k,h}\right)$ and $\bar{\mathbb{V}}_{2,k,h}\left(s_{k,h}, a_{k,h}\right)$:
\begin{align}
    \label{eqn: definition of E}
    E_{1,k,h} =& \min \left\{H^{2}, 4 H \check{\beta}_{k}\left\|\sum_{s'}\phi(s_{k,h},a_{k,h},s')\hat{V}_{k,h+1}^{\pi_k,\tilde{P}_{k}}(s', R_k) \right\|_{(\Lambda_{1,k,h})^{-1}}\right\}\\
    &+\min \left\{H^{2}, 2\tilde{\beta}_{k}\left\| \sum_{s'}\phi(s_{k,h},a_{k,h},s')\left(\hat{V}_{k,h+1}^{\pi_k,\tilde{P}_{k}}(s', R_k)\right)^2\right\|_{(\Lambda_{3,k,h})^{-1}}\right\}, \\
    E_{2,k,h} =& \min \left\{H^{2}, 4 H \check{\beta}_{k}\left\|\sum_{s'}\phi(s_{k,h},a_{k,h},s'){V}_{k,h+1}^{\pi_k,\tilde{P}_{k}}(s', R_k) \right\|_{(\Lambda_{2,k,h})^{-1}}\right\}\\
    &+\min \left\{H^{2}, 2\tilde{\beta}_{k}\left\| \sum_{s'}\phi(s_{k,h},a_{k,h},s')\left({V}_{k,h+1}^{\pi_k,\tilde{P}_{k}}(s', R_k)\right)^2\right\|_{(\Lambda_{4,k,h})^{-1}}\right\}.
\end{align}
We add the bonuses to $\bar{\mathbb{V}}_{1,k,h}(s_{k,h}, a_{k,h})$ and $\bar{\mathbb{V}}_{2,k,h}(s_{k,h}, a_{k,h})$ and maintain the optimistic variance estimation $\bar{\sigma}^2_{1,k,h}$ and $\bar{\sigma}^2_{2,k,h}$:
\begin{align}
    \label{eqn: definition of sigma}
    \bar{\sigma}^2_{1,k,h} &= {\max \left\{H^{2} / d,\bar{\mathbb{V}}_{1,k,h}(s_{k,h}, a_{k,h})+E_{1,k, h}\right\}}, \\
    \label{eqn: definition of sigma2}
    \bar{\sigma}^2_{2,k,h} &= {\max \left\{H^{2} / d,\bar{\mathbb{V}}_{2,k,h}(s_{k,h}, a_{k,h})+E_{2,k, h}\right\}}.
\end{align}
We use $\bar{\sigma}_{i,k,h}$ to normalize the value obtained in each step. By setting $\lambda = B^{-2}$ and solving the ridge regression defined in Eqn.~\ref{eqn: ridge regression eqn 1}, we know that $\hat{\theta}_{i,k+1,h}$ and $\Lambda_{i,k+1,h}$ are updated in the following way:
\begin{align}
    \label{eqn: definition of Lambda1}
    \Lambda_{1,k+1,h} &= \Lambda_{1,k,h} + (\bar{\sigma}_{1,k,h})^{-2}\left(\sum_{s'} \phi(s_{k,h},a_{k,h},s')\hat{V}^{\pi_k,\tilde{P}_k}_{k,h+1}(s', R_k)\right) \left(\sum_{s'} \phi(s_{k,h},a_{k,h},s')\hat{V}^{\pi_k,\tilde{P}_k}_{k,h+1}(s', R_k)\right)^{\top},\\
     \label{eqn: definition of Lambda2}
    \Lambda_{2,k+1,h} &= \Lambda_{1,k,h} + (\bar{\sigma}_{2,k,h})^{-2}\left(\sum_{s'} \phi(s_{k,h},a_{k,h},s'){V}^{\pi_k,\tilde{P}_k}_{k,h+1}(s', R_k)\right) \left(\sum_{s'} \phi(s_{k,h},a_{k,h},s'){V}^{\pi_k,\tilde{P}_k}_{k,h+1}(s', R_k)\right)^{\top}, \\
    \label{eqn: definition of theta1}
    \hat{\theta}_{1,k+1,h} &= (\Lambda_{1,k+1,h})^{-1} \sum_{t=1}^{k} (\bar{\sigma}_{1,t,h})^{-2}\left(\sum_{s'} \phi(s_{t,h},a_{t,h},s')\hat{V}^{\pi_{t},\tilde{P}_{t}}_{t,h+1}(s', R_{t})\right) \hat{V}^{\pi_{t},\tilde{P}_{t}}_{t,h+1}(s_{t,h+1}, R_{t}),\\
    \label{eqn: definition of theta2}
    \hat{\theta}_{2,k+1,h} &= (\Lambda_{2,k+1,h})^{-1} \sum_{t=1}^{k} (\bar{\sigma}_{2,t,h})^{-2}\left(\sum_{s'} \phi(s_{t,h},a_{t,h},s'){V}^{\pi_{t},\tilde{P}_{t}}_{t,h+1}(s', R_{t})\right) {V}^{\pi_{t},\tilde{P}_{t}}_{t,h+1}(s_{t,h+1}, R_{t}).
\end{align}

Similarly, we update the estimation w.r.t $\left(\hat{V}^{\pi_k,\tilde{P}_k}_{k,h+1}(s', R_k)\right)^2$ and $\left({V}^{\pi_k,\tilde{P}_k}_{k,h+1}(s', R_k)\right)^2$ :
\begin{align}
    \label{eqn: definition of Lambda3}
    \Lambda_{3,k+1,h} &= \Lambda_{3,k,h} + \left(\sum_{s'} \phi(s_{k,h},a_{k,h},s')\left(\hat{V}^{\pi_k,\tilde{P}_k}_{k,h+1}(s', R_k)\right)^2\right) \left(\sum_{s'} \phi(s_{k,h},a_{k,h},s')\left(\hat{V}^{\pi_k,\tilde{P}_k}_{k,h+1}(s', R_k)\right)^2\right)^{\top},\\
    \label{eqn: definition of Lambda4}
    \Lambda_{4,k+1,h} &= \Lambda_{4,k,h} + \left(\sum_{s'} \phi(s_{k,h},a_{k,h},s')\left({V}^{\pi_k,\tilde{P}_k}_{k,h+1}(s', R_k)\right)^2\right) \left(\sum_{s'} \phi(s_{k,h},a_{k,h},s')\left({V}^{\pi_k,\tilde{P}_k}_{k,h+1}(s', R_k)\right)^2\right)^{\top},\\
    \label{eqn: definition of theta3}
    \hat{\theta}_{3,k+1,h} &= (\Lambda_{3,k+1,h})^{-1} \sum_{t=1}^{k} \left(\sum_{s'} \phi(s_{t,h},a_{t,h},s')\left(\hat{V}^{\pi_{t},\tilde{P}_{t}}_{t,h+1}(s', R_{t})\right)^2\right) \left(\hat{V}^{\pi_{t},\tilde{P}_{t}}_{t,h+1}(s_{t,h+1}, R_{t})\right)^2,\\
    \label{eqn: definition of theta4}
    \hat{\theta}_{4,k+1,h} &= (\Lambda_{4,k+1,h})^{-1} \sum_{t=1}^{k} \left(\sum_{s'} \phi(s_{t,h},a_{t,h},s')\left({V}^{\pi_{t},\tilde{P}_{t}}_{t,h+1}(s', R_{t})\right)^2\right) \left({V}^{\pi_{t},\tilde{P}_{t}}_{t,h+1}(s_{t,h+1}, R_{t})\right)^2.
\end{align}

 We define $\tilde{Y}_{k,h}$ to be the ``value function'' with transition $\tilde{P}_k$ and reward $\bar{\mathbb{V}}_{1,k,h}$.
 \begin{align}
 \tilde{Y}_{k,H+1}(s)& = 0,\\
     \tilde{Y}_{k,h}(s) & = \bar{\mathbb{V}}_{1,k,h}(s,\pi_k(s)) + \tilde{P}_{k,h}\tilde{Y}_{k,h+1}(s,\pi_k(s)).
 \end{align}
$\hat{\theta}_{5,k,h}$ is the parameter estimation using samples w.r.t. $\tilde{Y}_{k,h}(s)$:
\begin{align}
    \label{eqn: definition of Lambda5}
    \Lambda_{5,k+1,h} &= \Lambda_{5,k,h} + \left(\sum_{s'} \phi(s_{k,h},a_{k,h},s')\tilde{Y}_{k,h+1}(s')\right) \left(\sum_{s'} \phi(s_{k,h},a_{k,h},s')\tilde{Y}_{k,h+1}(s')\right)^{\top},\\
    \label{eqn: definition of theta5}
    \hat{\theta}_{5,k+1,h} &= (\Lambda_{5,k+1,h})^{-1} \sum_{t=1}^{k} \left(\sum_{s'} \phi(s_{t,h},a_{t,h},s')\tilde{Y}_{t,h+1}(s')\right) \tilde{Y}_{t,h+1}(s_{k,h+1}).
\end{align}

We update the high-confidence set $\mathcal{U}_{k,h}$ by adding following five constraints: 
\begin{align}
    \label{inq: tilde theta constraint 1}
    \|\tilde{\theta}_h - \hat{\theta}_{1,k,h}\|_{\Lambda_{1,k,h}} &\leq \hat{\beta}, \\
    \label{inq: tilde theta constraint 2}
    \|\tilde{\theta}_h - \hat{\theta}_{2,k,h}\|_{\Lambda_{2,k,h}} &\leq \hat{\beta}, \\
    \label{inq: tilde theta constraint 3}
    \|\tilde{\theta}_h - \hat{\theta}_{3,k,h}\|_{\Lambda_{3,k,h}} &\leq \tilde{\beta}, \\
    \label{inq: tilde theta constraint 4}
    \|\tilde{\theta}_h - \hat{\theta}_{4,k,h}\|_{\Lambda_{4,k,h}} &\leq \tilde{\beta}, \\
    \label{inq: tilde theta constraint 5}
    \|\tilde{\theta}_h - \hat{\theta}_{5,k,h}\|_{\Lambda_{5,k,h}} &\leq \tilde{\beta}.
\end{align}
Note that we add new constraints to this confidence set $\mathcal{U}_{k,h}$ in each episode, instead of updating $\mathcal{U}_{k,h}$ as the intersection of the new constraints. This operation is designed to ensure that the cardinality of the confidence set $\mathcal{U}_{k,h}$ is non-increasing w.r.t. the episode $k$, so that ${V}_{k, 1}^{\pi, \tilde{P}}\left(s_{1}, R\right)$ is always non-increasing w.r.t. $k$.

During the proof, we use $\phi_{V^{\pi,P}}(s,a,R)$ as a shorthand of $\sum_{s'} \phi(s,a,s') V^{\pi,P}(s',R)$.

\subsection{High-probability events}

For notation convenience, we use $\mathbb{V}_{1,k,h}(s,a)$ and $\mathbb{V}_{2,k,h}(s,a)$ to denote the one-step transition variance with regard to $\hat{V}_{k,h+1}^{\pi_k,P_k}$ and ${V}_{k,h+1}^{\pi_k,P_k}$, i.e.
\begin{align}
    \mathbb{V}_{1,k,h}(s,a) & = \mathbb{E}_{s' \sim {P}_{h}(\cdot|s,a)}\left[\left(\hat{V}_{k,h+1}^{\pi_k,P_k}(s',R_k) - P_{h}\hat{V}_{k,h}^{\pi_k,P_k}(s,a,R_k)\right)^2\right], \\
    \mathbb{V}_{2,k,h}(s,a) & = \mathbb{E}_{s' \sim {P}_{h}(\cdot|s,a)}\left[\left({V}_{k,h+1}^{\pi_k,P_k}(s',R_k) - P_{h}{V}_{k,h}^{\pi_k,P_k}(s,a,R_k)\right)^2\right]. \\
\end{align}

\begin{lemma}
\label{lemma: high probability events}
With probability at least $1-\delta$, the following event holds for any $h \in [H], k \in [K]$:
\begin{align}
    \left\|\tilde{\theta}_{k,h} - \theta_h\right\|_{{{\Lambda}}_{1,k,h}} &\leq 2 \hat{\beta}, \\
    \left\|\tilde{\theta}_{k,h} - \theta_h\right\|_{{{\Lambda}}_{2,k,h}} &\leq 2\hat{\beta},  \\
    \left|\bar{\mathbb{V}}_{1,k,h}(s_{k,h},a_{k,h})-\mathbb{V}_{1,k,h}(s_{k,h},a_{k,h})\right| &\leq E_{1,k,h}, \\
    \left|\bar{\mathbb{V}}_{2,k,h}(s_{k,h},a_{k,h})-\mathbb{V}_{2,k,h}(s_{k,h},a_{k,h})\right| &\leq E_{2,k,h}, \\
     \left\|\tilde{\theta}_{k,h} - \theta_h\right\|_{{{\Lambda}}_{5,k,h}} &\leq 2\tilde{\beta}.
\end{align}

\end{lemma}
\begin{proof}
 We firstly prove the first and the third inequality, the second and the fourth inequality can be proved following the same idea. By the definition of $\bar{\mathbb{V}}_{1,k,h}(s_{k,h},a_{k,h})$ and ${\mathbb{V}}_{1,k,h}(s_{k,h},a_{k,h})$:
\begin{align}
\label{inq: high probability events 1}
    &\left|\bar{\mathbb{V}}_{1,k,h}(s_{k,h},a_{k,h}) -\mathbb{V}_{1,k,h}(s_{k,h},a_{k,h})\right| \\
    \leq & \min \left\{H^{2}, 2 H \left\| \tilde{\theta}_{k,h} - \theta_h\right\|_{\Lambda_{1,k,h}}\left\| {\phi}_{\hat{V}^{\pi_k,P_k}_{k,h+1}}\left(s_{k,h}, a_{k,h},R_k\right)\right\|_{(\Lambda_{1,k,h})^{-1}}\right\} \\
    &+\min \left\{H^{2},\left\| \tilde{\theta}_{k,h} - \theta_h\right\|_{\Lambda_{3,k,h}}\left\|  {\phi}_{\left(\hat{V}^{\pi_k,P_k}_{k, h+1}\right)^{2}}\left(s_{k,h}, a_{k,h},R_k\right)\right\|_{(\Lambda_{3,k,h})^{-1}}\right\}.
\end{align}
Let $x_k = (\bar{\sigma}_{1,k,h})^{-1} \phi_{\hat{V}^{\pi_k,P_k}_{k,h+1}}(s_{k,h},a_{k,h})$, and the noise 
$$\eta_k = (\bar{\sigma}_{1,k,h})^{-1} \hat{V}^{\pi_k,P_k}_{k, h+1}\left(s_{k,h+1}, R_k\right)-(\bar{\sigma}_{1,k,h})^{-1}\left\langle\boldsymbol{\phi}_{\hat{V}^{\pi_k,P_k}_{k, h+1}}\left(s_{k,h}, a_{k,h},R_k\right), \theta_{h}\right\rangle.$$ Since $\bar{\sigma}_{1,k,h} \geq \sqrt{H^2/d}$, we have $\|x_k\|_2 \leq \sqrt{d}$, $\mathbb{E}\left[\eta_{k}^{2} \mid \mathcal{G}_{k}\right] \leq d$. By Lemma~\ref{lemma: self-normalized bound, Bernstein}, we have with prob at least $1-\delta/(8H)$, for all $k \leq K$,
\begin{align}
    \left\|\theta_{h}-\hat{{\theta}}_{1,k,h}\right\|_{{{\Lambda}}_{1,k, h}} \leq 16d\sqrt{\log(1+KH^2/(d\lambda))\log(32K^2H/\delta)} + \sqrt{\lambda} B=\check{\beta}.
\end{align}
Similarly, we can prove that 
\begin{align}
   \left\|\theta_{h}-\hat{{\theta}}_{3,k,h}\right\|_{{{\Lambda}}_{3,k, h}} \leq 16H^2\sqrt{d\log(1+KH^4/(d\lambda))\log(32K^2H/\delta)} + \sqrt{\lambda} B=\tilde{\beta}.
\end{align}
Since $\left\|\tilde{{\theta}}_{k,h}-\hat{{\theta}}_{1,k,h}\right\|_{{{\Lambda}}_{1,k, h}} \leq \hat{\beta} \leq \check{\beta}$  by Inq~\ref{inq: tilde theta constraint 1}, we have $ \left\|\theta_{h}-\tilde{{\theta}}_{k,h}\right\|_{{{\Lambda}}_{1,k, h}}\leq 2\check{\beta}$. Similarly, since $\left\|\tilde{{\theta}}_{k,h}-\hat{{\theta}}_{3,k,h}\right\|_{{{\Lambda}}_{3,k, h}} \leq \tilde{\beta} $ by Inq~\ref{inq: tilde theta constraint 3}, we have $ \left\|\theta_{h}-\tilde{{\theta}}_{k,h}\right\|_{{{\Lambda}}_{3,k, h}}\leq 2
\tilde{\beta}$. Plugging the above inequalities back to Inq.~\ref{inq: high probability events 1}, we have
\begin{align}
    \left|\bar{\mathbb{V}}_{1,k,h}(s_{k,h},a_{k,h})-\mathbb{V}_{1,k,h}(s_{k,h},a_{k,h})\right| &\leq E_{1,k,h}.
\end{align}

The above inequality indicates that $\bar{\sigma}_{1,k,h} \geq \mathbb{V}_{1,k,h}$, which is an optimistic estimation. As a result, the variance of the single-step noise $\eta_k = (\bar{\sigma}_{1,k,h})^{-1} \hat{V}^{\pi_k,P_k}_{k, h+1}\left(s_{k,h+1},R_k\right)-(\bar{\sigma}_{ 1,k,h})^{-1}\left\langle\boldsymbol{\phi}_{\hat{V}^{\pi_k,P_k}_{ k,h+1}}\left(s_{k,h}, a_{k,h},R_k\right), {\theta}_{h}\right\rangle$ satisfies $\mathbb{E}\left[\eta_{k}^{2} \mid \mathcal{G}_{k}\right] \leq 1$. By Lemma~\ref{lemma: self-normalized bound, Bernstein}, we can prove a tighter confidence guarantee for $\hat{{\theta}}_{1,k,h}$:
\begin{align}
    \left\|{\theta}_{h}-\hat{{\theta}}_{1,k,h}\right\|_{{{\Lambda}}_{1,k,h}} \leq 16\sqrt{d\log(1+KH^2/(d\lambda))\log(32K^2H/\delta) }+ \sqrt{\lambda} B= \hat{\beta}.
\end{align}
Combining with $\left\|\tilde{{\theta}}_{1,k,h}-\hat{{\theta}}_{1,k,h}\right\|_{{{\Lambda}}_{1,k, h}} \leq \hat{\beta}$ by Inq~\ref{inq: tilde theta constraint 1}, we have $\left\|{\theta}_{h}-\tilde{{\theta}}_{1,k,h}\right\|_{{{\Lambda}}_{1,k,h}} \leq 2\hat{\beta}$.

Now we prove the last inequality in this lemma. Recall that we define 
 \begin{align}
 \tilde{Y}_{k,H+1}(s)& = 0,\\
     \tilde{Y}_{k,h}(s) & = \bar{\mathbb{V}}_{1,k,h}(s,\pi_k(s)) + \tilde{P}_{k,h}\tilde{Y}_{k,h+1}(s,\pi_k(s)).
 \end{align}
 Note that $ \tilde{Y}_{k,h}(s) \leq H^2$ by the law of total variance~\citep{lattimore2012pac,azar2013minimax}. Therefore, the variance on the single-step noise in at most $H^2$. By Lemma~\ref{lemma: self-normalized bound, Bernstein}, we have 
 \begin{align}
      \left\|{\theta}_{h}-\hat{{\theta}}_{5,k,h}\right\|_{{{\Lambda}}_{5,k,h}} \leq \tilde{\beta}.
 \end{align}
 Combining with $\left\|\tilde{\theta}_{h}-\hat{{\theta}}_{5,k,h}\right\|_{{{\Lambda}}_{5,k,h}} \leq \tilde{\beta}$ by Inq~\ref{inq: tilde theta constraint 5}, we have $\left\|{\theta}_{h}-\tilde{{\theta}}_{5,k,h}\right\|_{{{\Lambda}}_{5,k,h}} \leq 2\tilde{\beta}.$
\end{proof}

\begin{lemma}
\label{lemma: martingale difference, Berstein}
With probability at least $1-\delta/2$, we have 
\begin{align}
    \sum_{k=1}^{K}\sum_{h=1}^{H}\left(P_h{V}_{k,h+1}^{\pi_k,\tilde{P}_k}(s_{k,h},a_{k,h},R_k) - {V}_{k,h+1}^{\pi_k,\tilde{P}_k}(s_{k,h+1},a_{k,h+1},R_k) \right) \leq \sqrt{2H^3K \log(8/\delta)},\\
     \sum_{k=1}^{K}\sum_{h=1}^{H}\left(P_h\tilde{Y}_{k,h+1}(s_{k,h},a_{k,h},R_k) - \tilde{Y}_{k,h+1}(s_{k,h+1},a_{k,h+1},R_k) \right) \leq \sqrt{2H^5K \log(8/\delta)}.
\end{align}
\end{lemma}
\begin{proof}
 This lemma follows directly by Azuma's inequality for martingale difference sequence and union bound.
\end{proof}

During the following analysis, we denote the high-probability events defined in Lemma~\ref{lemma: high probability events} and Lemma~\ref{lemma: martingale difference, Berstein} as $\mathcal{E}_2$.

\subsection{Proof of Theorem~\ref{theorem: new main}}

\begin{lemma}
\label{lemma: new upper bound of sum V star}
Under event $\mathcal{E}_2$, we have 
\begin{align}
    &\sum_{k=1}^{K} V_{k,1}^{\pi_k,\tilde{P}_k}(s_1, R_k)\\ 
    \leq & O\left(\sqrt{(dH^4+d^2H^3)K \log(KH/\delta)\log^2(KH^4B^2)}) + (d^{2.5}H^2 + d^2H^3)  \log(KH/\delta)\log^2(KH^4B^2)\right). 
\end{align}
\end{lemma}
This lemma can be proved with the technique for regret analysis. We will explain it in Appendix~\ref{appendix: Berstein regret lemma}.

\begin{lemma}
\label{lemma: optimism, berstein case}
(Optimism) Under event $\mathcal{E}_2$, for any $h \in [H], k \in [K],s \in \mathcal{S}$, $\pi \in \Pi$, $R$ and $\tilde{P}$ satisfying $\tilde{P}_h = \tilde{\theta}_{h}^{\top} \phi$, we have
\begin{align}
    {V}_{k,h}^{\pi,\tilde{P}}(s,R) \geq \tilde{V}_{k,h}^{\pi,\tilde{P}}(s,R).
\end{align}
\end{lemma}
\begin{proof}
 This lemma is proved by induction. Suppose ${V}_{k,h+1}^{\pi,\tilde{P}}(s,R) \geq \tilde{V}_{k,h+1}^{\pi,\tilde{P}}(s,R)$, we have
 \begin{align}
     & {V}_{k,h}^{\pi,\tilde{P}}(s,R) - \tilde{V}_{k,h}^{\pi,\tilde{P}}(s,R) \\
     \geq & u_{2,k,h}^{\pi,\tilde{P}}(s,\pi(s),R) + (\tilde{P}_{k,h} - P_h) {V}_{k,h+1}^{\pi,\tilde{P}}(s,\pi(s),R) + {P}_h \left({V}_{k,h+1}^{\pi,\tilde{P}} - \tilde{V}_{k,h+1}^{\pi,\tilde{P}}\right)(s,\pi(s),R) \\
     \geq &u_{2,k,h}^{\pi,\tilde{P}}(s,\pi(s),R) + (\tilde{P}_{k,h} - P_h) {V}_{k,h+1}^{\pi,\tilde{P}}(s,\pi(s),R) \\
     \geq &u_{2,k,h}^{\pi,\tilde{P}}(s,\pi(s),R) - \left\|\tilde{\theta}_{k,h} - \theta_h\right\|_{\Lambda_{2,k,h}} \left\|\sum_{s'} \phi(s,\pi(s),s') {V}_{k,h+1}^{\pi, \tilde{P}}(s',R) \right\|_{(\Lambda_{2,k,h})^{-1}} \\
     \geq &0.
 \end{align}
 This indicates that ${V}_{k,h}^{\pi,\tilde{P}}(s,R) \geq \tilde{V}_{k,h}^{\pi,\tilde{P}}(s,R)$ holds for step $h$.
\end{proof}

\begin{lemma}
\label{lemma: regret summation to the last step}
under event $\mathcal{E}_2$, we have ${V}_{K,1}^{\pi_K, \tilde{P}_K}(s_1,R_K) \leq \frac{\sum_{k=1}^{K}V_{k,1}^{\pi_k,\tilde{P}_k}(s_1,R_k)}{K}$.
\end{lemma}

\begin{proof}
Firstly, we prove that ${V}_{k,1}^{\pi, \tilde{P}}(s,R)$ is non-increasing w.r.t $k$ for any fixed $\pi, \tilde{P}, s$ and $R$. This can be proved by induction. Suppose for any $k_1 \leq k_2$, ${V}_{k_1,h+1}^{\pi, \tilde{P}}(s,R) \geq {V}_{k_2,h+1}^{\pi, \tilde{P}}(s,R)$ for any $s$. Recall that 
\begin{align}
    {V}^{\pi,\tilde{P}}_{k,h}(s, R) &= \min\left\{ u_{1,k,h}^{\pi,\tilde{P}}(s,\pi_h(s), R) + u_{2,k,h}^{\pi,\tilde{P}}(s,\pi_h(s), R)+ \tilde{P}_h {V}^{\pi,\tilde{P}}_{k,h+1}(s,\pi_h(s),R),H\right\}.
\end{align}
Since $\Lambda_{1,k_1,h} \preccurlyeq \Lambda_{1,k_2,h}$, $\Lambda_{2,k_1,h} \preccurlyeq \Lambda_{2,k_2,h}$ and ${V}_{k_1,h+1}^{\pi, \tilde{P}}(s,R) \geq {V}_{k_2,h+1}^{\pi, \tilde{P}}(s,R)$ for any $s$, we can prove that 
\begin{align}
    u_{1,k_1,h}^{\pi,\tilde{P}}(s,\pi_h(s), R) &\geq u_{1,k_2,h}^{\pi,\tilde{P}}(s,\pi_h(s), R),\\
    u_{2,k_1,h}^{\pi,\tilde{P}}(s,\pi_h(s), R) & \geq u_{2,k_2,h}^{\pi,\tilde{P}}(s,\pi_h(s), R),\\
    \tilde{P}_h {V}^{\pi,\tilde{P}}_{k_1,h+1}(s,\pi_h(s),R) &\geq \tilde{P}_h {V}^{\pi,\tilde{P}}_{k_2,h+1}(s,\pi_h(s),R).
\end{align}
Therefore, ${V}_{k_1,h}^{\pi, \tilde{P}}(s,R) \geq {V}_{k_2,h}^{\pi, \tilde{P}}(s,R)$ holds for step $h$ and any $s$.

 Since the cardinality of the transition set $\mathcal{U}_{k,h}$ is non-increasing (We add more constraints in each episode), we know that $\tilde{\theta}_{k_2,h} \in \mathcal{U}_{k_1,h}$. By the optimality of $\pi_{k_1}, \tilde{\theta}_{k_1}$ and $R_{k_1}$ in episode $k_1$, we have ${V}_{k_1,1}^{\pi_{k_1}, \tilde{P}_{k_1}}(s_1,R_{k_1}) \geq {V}_{k_1,1}^{\pi_{k_2}, \tilde{P}_{k_2}}(s_1,R_{k_2})$. 
 
 Combining the above two inequalities, we have 
 \begin{align}
     {V}_{k_1,1}^{\pi_{k_1}, \tilde{P}_{k_1}}(s_1,R_{k,1}) \geq {V}_{k_1,1}^{\pi_{k_2}, \tilde{P}_{k_2}}(s_1,R_{k_2}) \geq {V}_{k_2,1}^{\pi_{k_2}, \tilde{P}_{k_2}}(s_1,R_{k_2}).
 \end{align}
This indicates that the value ${V}_{k,1}^{\pi_k, \tilde{P}_k}(s_1,R_k)$ is non-increasing w.r.t. $k$. Therefore, 
 \begin{align}
     K {V}_{K,1}^{\pi_K, \tilde{P}_K}(s_1,R_K) \leq \sum_{k=1}^{K} {V}_{k,1}^{\pi_k, \tilde{P}_k}(s_1,R_k).
 \end{align}
 
\end{proof}

\begin{lemma}
\label{lemma: sub-optimality gap during the planning phase} 
Under event $\mathcal{E}_2$, the sub-optimality gap of the policy $\hat{\pi}_R$ returned in the planning phase for any reward function $R$ can be bounded by
\begin{align}
    V^*_1(s_1,R) - V_1^{\hat{\pi}_R}(s_1,R)  & \leq 4 {V}^{\pi_K, \tilde{P}_{K}}_{K,1}(s_1,R_K) + \epsilon_{\rm opt}.
\end{align}
\end{lemma}

\begin{proof}
Recall that we use $\hat{V}^{\pi}_h$ to denote the value function of policy $\pi$ on the estimated model $\tilde{\theta}_{K}$.
\begin{align}
    \label{inq: sub-optimality gap 1}
     & V^*_1(s_1,R) - V_1^{\hat{\pi}_R}(s_1,R)  \\
     = & \left(V^*_1(s_1,R) -\hat{V}_1^{\pi^*_R, \tilde{P}_K}(s_1,R)\right) + \left(\hat{V}_1^{\hat{\pi}_R, \tilde{P}_K}(s_1,R) -V_1^{\hat{\pi}_R}(s_1) \right)+ \left(\hat{V}_1^{\pi^*_R, \tilde{P}_K}(s_1,R)- \hat{V}_1^{\hat{\pi}_R, \tilde{P}_K}(s_1,R) \right)\\
     \leq & V_1^*(s_1,R) -\hat{V}_1^{\pi^*_R, \tilde{P}_K}(s_1,R) + \hat{V}_1^{\hat{\pi}_R, \tilde{P}_K}(s_1,R) -V_1^{\hat{\pi}_R}(s_1,R) + \epsilon_{\rm opt}.
\end{align}
The inequality is because that the policy $\hat{\pi}_R$ is the $\epsilon_{\rm opt}$-optimal policy in the estimated MDP $\hat{M}$.

For notation convenience, we use $\operatorname{traj} \sim (\pi,P)$ to denote that the trajectory ($\{s_h,a_h\}_{h=1}^{H}$) is sampled with transition $P$ and policy $\pi$. For a certain sequence $\{R_h\}_{h=1}^{H}$, we define the function $W_{h}(\{R_h\})$ recursively from step $H+1$ to step $1$. Firstly, we define $W_{H+1}(\{R_h\}) = 0$. $W_h(R)$ is calculated from $W_{h+1}(R)$:
\begin{align}
    W_h(\{R_h\}) = \min\left\{H, R_h+  W_{h+1}(\{R_h\})\right\}.
\end{align}

With the above notation, we can prove that for any policy $\pi \in \Pi$,
\begin{align}
    & \left|\hat{V}^{\pi, \tilde{P}_{K}}_1(s_1, R) - V^{\pi}_1(s_1)\right| \\
    \leq & \left|\mathbb{E}_{\operatorname{traj} \sim (\pi,P)} W_1\left(\left\{(\tilde{P}_{K,h} - P_h)\hat{V}_{h+1}^{\pi, \tilde{P}_{K}}(s_h,a_h, R)\right\}\right)\right|  \\
    = & \left|\mathbb{E}_{\operatorname{traj} \sim (\pi,P)} W_1\left(\left\{(\tilde{\theta}_{K,h} - \theta_h)\sum_{s'}\phi(s_h,a_h,s')\hat{V}_{h+1}^{\pi, \tilde{P}_{K}}(s', R)\right\}\right)\right|\\
    \leq & \mathbb{E}_{\operatorname{traj} \sim (\pi,P)} W_1\left(\left\{\left\|\tilde{\theta}_{K,h} - \theta_h\right\|_{\Lambda_{1,K,h}} \left\|\sum_{s'}\phi(s_h,a_h,s')\hat{V}_{h+1}^{\pi, \tilde{P}_{K}}(s', R)\right\|_{\left(\Lambda_{1,K,h}\right)^{-1}}\right\}\right) \\
    \leq & \mathbb{E}_{\operatorname{traj} \sim (\pi,P)} W_1\left(\left\{2\hat{\beta} \left\|\sum_{s'}\phi(s_h,a_h,s')\hat{V}_{h+1}^{\pi, \tilde{P}_{K}}(s', R)\right\|_{\left(\Lambda_{1,K,h}\right)^{-1}} \right\}\right)  \\
    = &  2\mathbb{E}_{\operatorname{traj} \sim (\pi,P)} W_1\left(\left\{ u^{\pi,\tilde{P}_K}_{1,K,h}(s_h,a_h,R) \right\}\right) \\
    = & 2\tilde{V}^{\pi, \tilde{P}_{K}}_{K,1}(s_1,R) \\
    \leq & 2{V}^{\pi, \tilde{P}_{K}}_{K,1}(s_1,R) \\ 
    \leq & 2{V}^{\pi_K, \tilde{P}_{K}}_{K,1}(s_1,R_K).
\end{align}
The second inequality is due to lemma~\ref{lemma: high probability events} . The third inequality is due to Lemma~\ref{lemma: optimism, berstein case}. The last inequality is due to the optimality of $\pi_K$ and $R_K$. Plugging this inequality back to Inq.~\ref{inq: sub-optimality gap 1}, we can prove the lemma.
\end{proof}

\begin{proof}
(Proof of Theorem~\ref{theorem: new main}) Combining Lemma~\ref{lemma: new upper bound of sum V star}, Lemma~\ref{lemma: regret summation to the last step} and Lemma~\ref{lemma: sub-optimality gap during the planning phase}, we know that $K = \frac{C_2(d^2H^3+dH^4) \log(dH/(\delta\epsilon))\log^2(KHB/(\delta\epsilon))}{\epsilon^2} + \frac{C_2(d^{2.5}H^2 + d^2H^3) \log(dH/(\delta\epsilon))\log^2(KHB/(\delta\epsilon))}{\epsilon}$ for some constant $C_2$ suffices to guarantee that $V^*_1(s_1,R) - V_1^{\hat{\pi}}(s_1,R) \leq \epsilon+\epsilon_{\rm opt}$.
\end{proof}

\subsection{Proof of Lemma~\ref{lemma: new upper bound of sum V star}}
\label{appendix: Berstein regret lemma}

\begin{proof}

From the standard value decomposition technique~\citep{azar2017minimax,jin2018q}, we can decompose ${V}_{k,1}^{\pi_k, \tilde{P}_k}(s_1,R_k)$ into following terms:
\begin{align}
    & {V}_{k,1}^{\pi_k, \tilde{P}_k}(s_1,R_k) \\
    \leq & \sum_{h=1}^{H} \min\left\{H, u^{\pi_k, \tilde{P}_k}_{1,k,h}(s_{k,h},a_{k,h}, R_k) \right\} + \sum_{h=1}^{H}\min\left\{H, u_{2,k,h}^{\pi_k,\tilde{P}_k}(s_{k,h},a_{k,h}, R_k) \right\} \\
    &+ \sum_{h=1}^{H} \min\left\{H,(\tilde{P}_{k,h} - P_h){V}_{k,h+1}^{\pi_k, \tilde{P}_k}(s_{k,h},a_{k,h},R_k)\right\} \\
    &+ \sum_{h=1}^{H}\left(P_h{V}_{k,h+1}^{\pi_k,\tilde{P}_k}(s_{k,h},a_{k,h},R_k) - {V}_{k,h+1}^{\pi_k, \tilde{P}_k}(s_{k,h+1},a_{k,h+1},R_k) \right) \\
    \label{eqn: Bernstein regret upper bound, part 1, u1}
    \leq & \sum_{h=1}^{H}  \min\left\{H,u^{\pi_k, \tilde{P}_k}_{1,k,h}(s_{k,h},a_{k,h}, R_k)\right\} \\
    \label{eqn: Bernstein regret upper bound, part 1, u2}
    &+ 2\sum_{h=1}^{H}\min\left\{H, u^{\pi_k, \tilde{P}_k}_{2,k,h}(s_{k,h},a_{k,h},  R_k)\right\} \\
    \label{eqn: Bernstein regret upper bound, part 2}
    &+ \sum_{h=1}^{H}\left(P_h{V}_{k,1}^{\pi_k,\tilde{P}_k}(s_{k,h},a_{k,h},R_k) - {V}_{k,1}^{\pi_k,\tilde{P}_k}(s_{k,h+1},a_{k,h+1},R_k) \right),
\end{align}
where the second inequality is derived from
\begin{align}
     (\tilde{P}_{k,h} - P_h){V}_{k,h+1}^{\pi_k, \tilde{P}_k}(s_{k,h},a_{k,h},R_k) 
    =& (\tilde{\theta}_{k,h} - \theta_{h}) \sum_{s'} \phi(s_{k,h},a_{k,h},s') {V}_{k,h+1}^{\pi_k, \tilde{P}_k}(s',R_k) \\
    \leq & \left\|\tilde{\theta}_{k,h} - \theta_{h}\right\|_{\Lambda_{2,k,h}} \left\|\sum_{s'}\phi(s_{k,h},a_{k,h},s') {V}_{k,h+1}^{\pi_k, \tilde{P}_k}(s',R_k)\right\|_{\Lambda^{-1}_{2,k,h}} \\
    \leq &u^{\pi_k, \tilde{P}_k}_{2,k,h}(s_{k,h},a_{k,h},  R_k).
\end{align}

 Eqn~\ref{eqn: Bernstein regret upper bound, part 2} is a martingale difference sequence. By Lemma~\ref{lemma: martingale difference, Berstein}, the summation over all $k \in [K]$ is at most $\sqrt{2H^3K \log(4/\delta)}$. We mainly focus on Eqn~\ref{eqn: Bernstein regret upper bound, part 1, u1} and Eqn~\ref{eqn: Bernstein regret upper bound, part 1, u2}.
 
 \paragraph{Upper bound of Eqn~\ref{eqn: Bernstein regret upper bound, part 1, u1}} Firstly, we bound the summation of $\min\left\{H,u^{\pi_k, \tilde{P}_k}_{1,k,h}(s_{k,h},a_{k,h},  R_k)\right\}$.
 
 \begin{align}
 \label{eqn: upper bound of summation u1kh}
     &\sum_{k=1}^{K} \sum_{h=1}^{H} \min\left\{H,u^{\pi_k, \tilde{P}_k}_{1,k,h}(s_{k,h},a_{k,h},  R_k)\right\} \\
     = & \sum_{k=1}^{K} \sum_{h=1}^{H}\min\left\{H,\hat{\beta} \left\|\sum_{s'} \phi(s_{k,h},a_{k,h},s') \hat{V}^{\pi_k,\tilde{P}_k}_{k,h+1}(s',R_k)\right\|_{\Lambda^{-1}_{1,k,h}}\right\} \\
     \leq & \sum_{k=1}^{K} \sum_{h=1}^{H}\hat{\beta} \bar{\sigma}_{1,k,h}\min\left\{1,\left\|\sum_{s'} \phi(s_{k,h},a_{k,h},s') \frac{\hat{V}^{\pi_k,\tilde{P}_k}_{k,h+1}(s',R_k)}{\bar{\sigma}_{1,k,h}}\right\|_{\Lambda^{-1}_{1,k,h}}\right\} \\
     \leq & \hat{\beta} \sqrt{\sum_{k,h} \bar{\sigma}^2_{1,k,h} \sum_{k,h}\min\left\{1,\left\|\sum_{s'} \phi(s_{k,h},a_{k,h},s') \frac{\hat{V}^{\pi_k,\tilde{P}_k}_{k,h+1}(s',R_k)}{\bar{\sigma}_{1,k,h}}\right\|^2_{\Lambda^{-1}_{1,k,h}}\right\}} \\
     \leq & \hat{\beta} \sqrt{\sum_{k,h} \bar{\sigma}^2_{1,k,h} \cdot 2dH \log(1+KH/\lambda)}.
 \end{align}
 The first inequality is due to Cauchy-Schwarz inequality. The second inequality is due to Lemma~\ref{lemma: data accumulation}.
 
 By the definition of $\bar{\sigma}^2_{1,k,h}$, we have 
 \begin{align}
    \label{inq: upper bound of sigma}
     \bar{\sigma}^2_{1,k,h} \leq H^2/d + E_{1,k,h} + \bar{\mathbb{V}}_{1,k,h}(s_{k,h},a_{k,h}).
 \end{align}
 Now we bound $\sum_{k,h}E_{1,k,h}$ and $\sum_{k,h}\bar{\mathbb{V}}_{1,k,h}(s_{k,h},a_{k,h})$ respectively. For $\sum_{k,h}E_{1,k,h}$, we have
 \begin{align}
     \sum_{k,h} E_{1,k,h} \leq & \sum_{h=1}^{H} \sum_{k=1}^{K} \min \left\{H^2, 4 H \check{\beta}\left\|\sum_{s'}\phi(s_{k,h},a_{k,h},s')\hat{V}_{k,h+1}^{\pi_k,\tilde{P}_{k}}(s', R_k) \right\|_{(\Lambda_{1,k,h})^{-1}} \right\} \\
     &+ \sum_{h=1}^{H} \sum_{k=1}^{K}\min \left\{H^2,2\tilde{\beta}\left\| \sum_{s'}\phi(s_{k,h},a_{k,h},s')\left(\hat{V}_{k,h+1}^{\pi_k,\tilde{P}_{k}}(s', R_k)\right)^2\right\|_{(\Lambda_{3,k,h})^{-1}} \right\}.
 \end{align}
 For the first part,
 \begin{align}
 \label{inq: upper bound of sigma, part 1}
     &\sum_{h=1}^{H} \sum_{k=1}^{K}  \min \left\{H^2,4 H \check{\beta}\left\|\sum_{s'}\phi(s_{k,h},a_{k,h},s')\hat{V}_{k,h+1}^{\pi_k,\tilde{P}_{k}}(s', R_k) \right\|_{(\Lambda_{1,k,h})^{-1}} \right\}\\
     \leq & 4H \sum_{h=1}^{H} \sum_{k=1}^{K}  \check{\beta} \bar{\sigma}_{1,k,h}\min \left\{1,\left\|\sum_{s'}\phi(s_{k,h},a_{k,h},s')\hat{V}_{k,h+1}^{\pi_k,\tilde{P}_{k}}(s', R_k)/\bar{\sigma}_{1,k,h} \right\|_{(\Lambda_{1,k,h})^{-1}}\right\} \\
     \leq & 4H \check{\beta} \sqrt{\sum_{k,h} \bar{\sigma}^2_{1,k,h} \sum_{k,h}\min \left\{1,\left\|\sum_{s'} \phi(s_{k,h},a_{k,h},s') \frac{\hat{V}^{\pi_k,\tilde{P}_k}_{k,h+1}(s',R_k)}{\bar{\sigma}_{1,k,h}}\right\|^2_{\Lambda^{-1}_{1,k,h}}\right\}} \\
     \leq & 4H\check{\beta} \sqrt{\sum_{k,h} \bar{\sigma}^2_{1,k,h} \cdot 2dH \log(1+KH/\lambda)}, 
 \end{align}
 where the last inequality is due to Lemma~\ref{lemma: data accumulation}. Similarly, for the second part, by Lemma~\ref{lemma: data accumulation}, we have
 \begin{align}
 \label{inq: upper bound of sigma, part 2}
      &\sum_{h=1}^{H} \sum_{k=1}^{K}\min\left\{H^2,2\tilde{\beta}\left\| \sum_{s'}\phi(s_{k,h},a_{k,h},s')\left(\hat{V}_{k,h+1}^{\pi_k,\tilde{P}_{k}}(s', R_k)\right)^2\right\|_{(\Lambda_{3,k,h})^{-1}}\right\} \\
     \leq & 2\tilde{\beta} \sqrt{2dH^2K \log(1+KH/\lambda)}. 
 \end{align}
 
 Note that $\bar{\mathbb{V}}_{1,k,h}(s_{k,h},a_{k,h}) $ is the empirical variance of $\hat{V}_{k,h+1}^{\pi_k,\tilde{P}_{k}}(s', R_k)$ with transition $\tilde{P}_k(s'|s_{k,h},a_{k,h})$. We bound the summation of $\bar{\mathbb{V}}_{1,k,h}(s_{k,h},a_{k,h}) $ by the law of total variance~\citep{lattimore2012pac,azar2013minimax}.
 
 Recall that we define 
 \begin{align}
 \tilde{Y}_{k,H+1}(s)& = 0,\\
     \tilde{Y}_{k,h}(s) & = \bar{\mathbb{V}}_{1,k,h}(s,\pi_k(s)) + \tilde{P}_{k,h}\tilde{Y}_{k,h+1}(s,\pi_k(s)).
 \end{align}
 By the law of total variance, we have $\tilde{Y}_{k,1}(s) \leq H^2$ holds for any $s,a$ and $k$. We now bound the difference between  $\sum_{k=1}^{K}\tilde{Y}_{k,1}(s_{k,1})$ and $\sum_{k=1}^{K}\sum_{h=1}^{H} \bar{\mathbb{V}}_{1,k,h}(s_{k,h},a_{k,h})$.
 \begin{align}
    & \tilde{Y}_{k,1}(s_{k,1}) - \sum_{h=1}^{H} \bar{\mathbb{V}}_{1,k,h}(s_{k,h},a_{k,h}) \\
    = &  \tilde{P}_{k,1}\tilde{Y}_{k,2}(s_{k,1},a_{k,1}) - \sum_{h=2}^{H} \bar{\mathbb{V}}_{1,k,h}(s_{k,h},a_{k,h}) \\
    \leq & \min\left\{H^2, \left( \tilde{P}_{k,1} - P_{1}\right)\tilde{Y}_{k,2}(s_{k,1},a_{k,1})\right\} + \left(P_{1}\tilde{Y}_{k,2}(s_{k,1},a_{k,1}) - \tilde{Y}_{k,2}(s_{k,2}) \right) \\
    & + \left(\tilde{Y}_{k,2}(s_{k,2}) - \sum_{h=2}^{H} \bar{\mathbb{V}}_{1,k,h}(s_{k,h},a_{k,h})\right).
 \end{align}
 Therefore, we have 
 \begin{align}
    & \sum_{k=1}^{K}\tilde{Y}_{k,1}(s_{k,1}) - \sum_{k=1}^{K}\sum_{h=1}^{H} \bar{\mathbb{V}}_{1,k,h}(s_{k,h},a_{k,h}) \\
    \label{eqn: tildeY 1}
     \leq & \sum_{k=1}^{K}\sum_{h=1}^{H} \min\left\{H^2,\left( \tilde{P}_{k,h} - P_{h}\right)\tilde{Y}_{k,h+1}(s_{k,h},a_{k,h})\right\} \\
     \label{eqn: tildeY 2}
     &+ \sum_{h=1}^{H} \sum_{k=1}^{K}\left(P_{h}\tilde{Y}_{k,h+1}(s_{k,h},a_{k,h}) - \tilde{Y}_{k,h+1}(s_{k,h+1}) \right).
 \end{align}
 For Eqn~\ref{eqn: tildeY 2}, this term can be regarded as a martingale difference sequence, thus can be bounded by $H^2\sqrt{KH}$ by Lemma~\ref{lemma: martingale difference, Berstein}. For Eqn~\ref{eqn: tildeY 1}, we can bound this term in the following way:
 \begin{align}
     &\sum_{k=1}^{K}\sum_{h=1}^{H} \min\left\{H^2,\left( \tilde{P}_{k,h} - P_{h}\right)\tilde{Y}_{k,h+1}(s_{k,h},a_{k,h})\right\} \\
     \leq & \sum_{k=1}^{K}\sum_{h=1}^{H}\min\left\{H^2,\left\|\tilde{\theta}_{k,h} - \theta_{h}\right\|_{\Lambda_{5,k,h}} \left\|\phi_{\tilde{Y}}(s_{k,h},a_{k,h},R)\right\|_{\Lambda^{-1}_{5,k,h}}\right\} \\
     \leq & \tilde{\beta} \sum_{k=1}^{K}\sum_{h=1}^{H} \min\left\{1,\left\|\phi_{\tilde{Y}}(s_{k,h},a_{k,h},R)\right\|_{\Lambda^{-1}_{5,k,h}}\right\} \\ 
     \leq & \tilde{\beta} \sum_{h=1}^{H} \sqrt{K  \sum_{k=1}^{K}\min\left\{1,\left\|\phi_{\tilde{Y}}(s_{k,h},a_{k,h},R)\right\|^2_{\Lambda^{-1}_{5,k,h}}\right\}} \\
     \leq & \tilde{\beta} \sqrt{2dH^2K \log(1+KH/\lambda)}.
 \end{align}
 The second inequality is due to Lemma~\ref{lemma: high probability events}. The third inequality is due to Cauchy-Schwarz inequality. The last inequality is due to Lemma~\ref{lemma: data accumulation}.
 
 From the above analysis, we have 
 \begin{align}
 \label{eqn: variance bound}
     \sum_{k,h} \bar{\mathbb{V}}_{1,k,h}(s_{k,h},a_{k,h}) \leq H^2K + \tilde{\beta} \sqrt{2dH^2K \log(1+KH/\lambda)}.
 \end{align}
 Therefore, the summation of $\bar{\sigma}_{1,k,h}^2$ can be bounded as 
 \begin{align}
     &\sum_{k,h} \bar{\sigma}^2_{1,k,h} \\
     \leq & H^3K/d + \sum_{k,h} E_{1,k,h} + \sum_{k,h} \bar{\mathbb{V}}_{1,k,h} \\
     \leq & H^3K/d + 4H\check{\beta} \sqrt{\sum_{k,h} \bar{\sigma}^2_{1,k,h} \cdot 2dH \log(1+KH/\lambda)} +  3\tilde{\beta} \sqrt{ dH^2K \log(1+KH/\lambda)} + H^2K.
 \end{align}
 We define $\tau = \log(32K^2H/\delta)\log^2(1+KH^4B^2)$. Solving for $\sum_{k,h} \bar{\sigma}^2_{1,k,h}$, we have 
 \begin{align}
     \sum_{k,h} \bar{\sigma}^2_{1,k,h} 
     \leq& c_1\left(H^3K/d + H^2K + H^3d^3 \tau
     + \sqrt{H^2d^3K \tau}\right) \\
      \leq& 2c_1 \left(H^3K/d + H^2K + H^3d^3 \tau\right),
 \end{align}
 where $c_1$ denote a certain constant. The last inequality is due to $2\sqrt{ab} \leq a+b$ for  $a,b \geq 0$. Plugging the above inequality back to Inq~\ref{eqn: upper bound of summation u1kh}, we have for a constant $c_2$, 
 \begin{align}
    \label{inq: upper bound of u1}
     &\sum_{k=1}^{K} \sum_{h=1}^{H} u^{\pi_k, \tilde{P}_k}_{1,k,h}(s_{k,h},a_{k,h}, R_k) 
     \leq c_2 \sqrt{\left(dH^4+d^2H^3\right)K \tau} + c_2 d^{2.5}H^2 \tau.
 \end{align}
 
 \paragraph{Upper bound of Eqn~\ref{eqn: Bernstein regret upper bound, part 1, u2}} Now we focus on the summation of $u^{\pi_k, \tilde{P}_k}_{2,k,h}(s_{k,h},a_{k,h},  R_k)$. The proof follows almost the same arguments as the summation of $u^{\pi_k, \tilde{P}_k}_{1,k,h}(s_{k,h},a_{k,h},  R_k)$, though the upper bound of $\sum_{k=1}^K \sum_{h=1}^{H}\bar{\mathbb{V}}_{2,k,h}(s_{k,h},a_{k,h})$ is derived in a different way. Following the same proof idea, we can show that 
 \begin{align}
    \label{eqn: u2 eqn 1}
     &\sum_{k=1}^K \sum_{h=1}^{H} u_{2,k,h}^{\pi_k,\tilde{P}_k}(s_{k,h},a_{k,h},R_k ) \leq  \hat{\beta}\sqrt{\sum_{k,h} \bar{\sigma}^2_{2,k,h} 2dH \log(1+KH/\lambda)}, \\
     \label{eqn: u2 eqn 2}
     &\bar{\sigma}^2_{2,k,h} \leq  H^2/d + E_{2,k,h} + \bar{\mathbb{V}}_{2,k,h} (s_{k,h},a_{k,h}),  \\
     \label{eqn: u2 eqn 3}
      &\sum_{k=1}^{K} \sum_{h=1}^{H} E_{2,k,h} \leq 4 H \check{\beta} \sqrt{\sum_{k, h} \bar{\sigma}_{2, k, h}^{2} \cdot 2 d H \log (1+KH/\lambda)} + 2 \tilde{\beta} \sqrt{ 2d H^2K \log (1+KH/\lambda)}.
 \end{align}
 
 Combining Inq~\ref{eqn: u2 eqn 2} and Inq~\ref{eqn: u2 eqn 3} and solving for $\sum_{k,h}\bar{\sigma}^2_{2,k,h}$, we have 
 \begin{align}
     &\sum_{k,h}\bar{\sigma}^2_{2,k,h} \\
     \leq & 2H^3K/d + 16dH^3 \check{\beta}^2 \log(1+KH/\lambda) + 4 \tilde{\beta}\sqrt{ dH^2K\log(1+KH/\lambda)} +  2\sum_{k,h}\bar{\mathbb{V}}_{2,k,h} (s_{k,h},a_{k,h}) \\
     \leq & 6H^3K/d + 16dH^3 \check{\beta}^2 \log(1+KH/\lambda) + 4 \tilde{\beta}^2 d^2/H\log(1+KH/\lambda) +  2\sum_{k,h}\bar{\mathbb{V}}_{2,k,h} (s_{k,h},a_{k,h})
 \end{align}
 Plugging this inequality back to Inq~\ref{eqn: u2 eqn 1}, we have
 \begin{align}
    \label{eqn: u2 eqn 4}
      &\sum_{k=1}^K \sum_{h=1}^{H} u_{2,k,h}^{\pi_k,\tilde{P}_k}(s_{k,h},a_{k,h},R_k ) \\
     \leq & \hat{\beta} \sqrt{6H^4K \log(1+KH/\lambda)} + 4\check{\beta}\hat{\beta}dH^2 \log (1+KH/\lambda) + 2 \tilde{\beta}\hat{\beta}d^{3/2}\log (1+KH/\lambda) \\
     &+ \sqrt{32d^2H \log(32K^2H/\delta)\log^2(1+KH^4B^2)\sum_{k,h}\bar{\mathbb{V}}_{2,k,h} (s_{k,h},a_{k,h}) }.
 \end{align}
 
 We now bound the last term in the above equation. Define $\tau = \log(32K^2H/\delta)\log^2(1+KH^4B^2)$, we have:
 \begin{align}
     &\sqrt{32d^2H \tau\sum_{k,h}\bar{\mathbb{V}}_{2,k,h} (s_{k,h},a_{k,h}) } \\
     \leq & \sqrt{32d^2H \tau\sum_{k=1}^K \sum_{h=1}^{H} \mathbb{E}_{s' \sim \tilde{P}_{k,h}(\cdot|s_{k,h},a_{k,h})}\left[(V^{\pi_{k}, \tilde{P}_{k}}_{k,h+1})^2(s',R_k)\right]} \\
     \leq & \sqrt{32d^2H^2 \tau \sum_{k=1}^K \sum_{h=1}^{H} \mathbb{E}_{s' \sim \tilde{P}_{k,h}(\cdot|s_{k,h},a_{k,h})}\left[V^{\pi_{k}, \tilde{P}_{k}}_{k,h+1}(s',R_k)\right]} \\
     \leq & 4\sqrt{2}\tilde{c}_3d^2H^3 \tau + \frac{1}{\tilde{c}_3H}\sum_{k=1}^K \sum_{h=1}^{H} \mathbb{E}_{s' \sim \tilde{P}_{k,h}(\cdot|s_{k,h},a_{k,h})}\left[V^{\pi_{k}, \tilde{P}_{k}}_{k,h+1}(s',R_k)\right] \\
     = & 4\sqrt{2}\tilde{c}_3d^2H^3 \tau + \frac{1}{\tilde{c}_3H}\sum_{k=1}^K \sum_{h=1}^{H} V^{\pi_{k}, \tilde{P}_{k}}_{k,h+1}(s_{k,h+1},R_k) \\
     \label{eqn: variance V2, part1}
     & + \frac{1}{\tilde{c}_3H} \sum_{k=1}^K \sum_{h=1}^{H} \left(P_hV^{\pi_{k}, \tilde{P}_{k}}_{k,h+1}(s_{k,h},a_{k,h},R_k) - V^{\pi_{k}, \tilde{P}_{k}}_{k,h+1}(s_{k,h+1},R_k)\right) \\
      \label{eqn: variance V2, part2}
     & + \frac{1}{\tilde{c}_3H} \sum_{k=1}^K \sum_{h=1}^{H} \left(\tilde{P}_{k,h} - P_h\right)V^{\pi_{k}, \tilde{P}_{k}}_{k,h+1}(s_{k,h},a_{k,h},R_k),
 \end{align}
 where $\tilde{c}_3 \geq 1$ is a constant to be defined later. The last inequality is due to $2\sqrt{ab} \leq a + b$ for any $a,b \geq 0$. Note that by Lemma~\ref{lemma: high probability events}, Eqn~\ref{eqn: variance V2, part2} is upper bounded by $\frac{1}{c_3H} \sum_{k,h} u^{\pi_k, \tilde{P}_k}_{2,k,h}(s_{k,h},a_{k,h},  R_k)$. Eqn~\ref{eqn: variance V2, part1} is a martingale difference sequence, and can be bounded by $\sqrt{HK \log(1/\delta)}$ with high probability. Plugging the above inequality back to Inq~\ref{eqn: u2 eqn 4} and solving for $\sum_{k=1}^K \sum_{h=1}^{H} u_{2,k,h}^{\pi_k,\tilde{P}_k}(s_{k,h},a_{k,h},R_k)$, we have 
 \begin{align}
     & \sum_{k=1}^K \sum_{h=1}^{H} u_{2,k,h}^{\pi_k,\tilde{P}_k}(s_{k,h},a_{k,h},R_k) \\
     \leq & c_3 \left(\sqrt{dH^4K \tau} + (d^{2.5}H^2+ d^2H^3) \tau + \frac{1}{\tilde{c}_3H} \sum_{k=1}^{K} \sum_{h=1}^{H} V^{\pi_{k}, \tilde{P}_{k}}_{k,h+1}(s_{k,h+1},R_k)\right),
 \end{align}
 where $c_3 \geq 1$ is a constant here. We set $\tilde{c}_3 = c_3$, then we have 
 
  \begin{align}
    \label{inq: upper bound of u2}
     & \sum_{k=1}^K \sum_{h=1}^{H} u_{2,k,h}^{\pi_k,\tilde{P}_k}(s_{k,h},a_{k,h},R_k) \\
     \leq & c_3 \left(\sqrt{dH^4K \tau} + (d^{2.5}H^2+ d^2H^3)\tau\right) + \frac{1}{H} \sum_{k=1}^{K} \sum_{h=1}^{H} V^{\pi_{k}, \tilde{P}_{k}}_{k,h+1}(s_{k,h+1},R_k),
 \end{align}
\paragraph{Bounding the summation of $V_{k, 1}^{\pi_{k}, \tilde{P}_{k}}\left(s_{1}, R_{k}\right)$} In the above analysis, we derive the upper bound of Eqn~\ref{eqn: Bernstein regret upper bound, part 1, u1} and Eqn~\ref{eqn: Bernstein regret upper bound, part 1, u2} (Inq~\ref{inq: upper bound of u1} and Inq~\ref{inq: upper bound of u2}). Now we can finally bound $\sum_{k=1}^{K} V_{k, 1}^{\pi_{k}, \tilde{P}_{k}}\left(s_{k,1}, R_{k}\right)$. For a constant $c_4$,
\begin{align}
    & \sum_{k=1}^{K} V_{k, 1}^{\pi_{k}, \tilde{P}_{k}}\left(s_{k,1}, R_{k}\right) \\
    \leq &c_4\sqrt{(dH^4+d^2H^3)K \tau}) + c_4(d^{2.5}H^2 + d^2H^3)\tau + \frac{1}{H} \sum_{k=1}^{K} \sum_{h=1}^{H} V^{\pi_{k}, \tilde{P}_{k}}_{k,h+1}(s_{k,h+1},R_k).
\end{align}

Following the same analysis, the above inequality actually also holds for any step $h \in [H]$:
\begin{align}
    & \sum_{k=1}^{K} V_{k, h}^{\pi_{k}, \tilde{P}_{k}}\left(s_{k,h}, R_{k}\right) \\
    \leq &c_4\sqrt{(dH^4+d^2H^3)K \tau}) + c_4(d^{2.5}H^2 + d^2H^3) \tau + \frac{1}{H} \sum_{k=1}^{K} \sum_{h_1=h}^{H} V^{\pi_{k}, \tilde{P}_{k}}_{k,h_1+1}(s_{k,h_1+1},R_k).
\end{align}

Define $G = c_4\left(\sqrt{(dH^4+d^2H^3)K \tau}) + c_4(d^{2.5}H^2 + d^2H^3)\tau\right)$ and $a_h = \sum_{k=1}^{K} V_{k, h}^{\pi_{k}, \tilde{P}_{k}}\left(s_{k,h}, R_{k}\right)$. The above inequality can be simplified into the following form:
\begin{align}
    a_h \leq G + \frac{1}{H} \sum_{h_1 = h}^{H} a_{h_1},
\end{align}
and we have $a_{H+1} = 0$. From the elementary calculation, we can prove that $a_1 \leq (1+\frac{1}{H})^H G \leq e G$. Therefore, we have 
\begin{align}
    &\sum_{k=1}^{K} V_{k, 1}^{\pi_{k}, \tilde{P}_{k}}\left(s_{k,1}, R_{k}\right) \\
    \leq &c_5\left(\sqrt{(dH^4+d^2H^3)K \log(KH/\delta)\log^2(KH^4B^2)}) + (d^{2.5}H^2 + d^2H^3)  \log(KH/\delta)\log^2(KH^4B^2)\right).
\end{align}
for a constant $c_5 = e c_4$.
\end{proof}

\section{The Lower Bound for Reward-free Exploration}
\label{appendix: lower bound}
In this section, we prove that even in the setting of non-plug-in reward-free exploration, the sample complexity in the exploration phase to obtain an $\epsilon$-optimal policy is at least $\Omega(\frac{d^2H^3}{\epsilon^2})$. We say an algorithm can $(\epsilon,\delta)$-learns the linear mixture MDP $\mathcal{M}$ if this algorithm returns an $\epsilon$-optimal policy for $\mathcal{M}$ with probability at least $1-\delta$ in the planning phase after receiving samples for $K$ episodes in the exploration phase. Our theorem is stated as follows.

\begin{theorem}
\label{thm: lower bound}
Suppose $\epsilon \leq \min\left(C_1\sqrt{H}, C_2\sqrt{dH^4B}\right), B > 1, H > 4, d > 3$ for certain positive constants $C_1$ and $C_2$. Then for any algorithm $\mathcal{ALG}$ there exists an linear mixture MDP instance $\mathcal{M} = (\mathcal{S}, \mathcal{A}, P,R,H,\nu)$ such that if $\mathcal{ALG}$  $(\epsilon,\delta)$-learns the problem $\mathcal{M}$, $\mathcal{ALG}$ needs to collect at least $K = Cd^2H^3/\epsilon^2$ episodes during the exploration phase, where $C$ is an absolute constant, and $0 < \delta < 1$ is a positive constant that has no dependence on $\epsilon, H, d, K$.
\end{theorem}

Compared with the lower bound proposed by \cite{zhang2021reward}, we further improve their result by a factor of $d$. This lower bound also indicates that our sample complexity bound in Theorem~\ref{theorem: new main} is statistically optimal. We prove Theorem~\ref{thm: lower bound} in Appendix~\ref{appendix: proof of the lower bound}.

\subsection{Proof of Theorem~\ref{thm: lower bound}}
\label{appendix: proof of the lower bound}
Our basic idea to prove Theorem~\ref{thm: lower bound} is to connect the sample complexity lower bound with the regret lower bound of another constructed learning algorithm in the standard online exploration setting. 

To start with, we notice that the reward-free exploration is strictly harder than the standard non-reward-free online exploration setting (e.g. \cite{jin2018q,zhou2021provably,zhou2020nearly}), where the reward function is deterministic and known during the exploration. Therefore, if we can prove the sample complexity lower bound in the standard online exploration setting, this bound can also be applied to the reward-free setting.

For readers who are not familiar with the formulation of online exploration setting studied in this section, we firstly introduce the preliminaries. Compared with the reward-free exploration, the only difference is that the reward $R(s,a)$ is fixed and known to the agent. In each episode, the agent starts from an initial state $s_{k,1}$ sampled from the distribution $\nu$. At each step $h \in [H]$, the agent observes the current state $s_{k,h} \in \mathcal{S}$, takes action $a_{k,h} \in \mathcal{A}$, receives the deterministic reward $R_h(s_{k,h},a_{k,h})$, and transits to state $s_{k,h+1}$ with probability $P_{h}(s_{k,h+1}|s_{k,h},a_{k,h})$. The episode ends when $s_{H+1}$ is reached. The agent's goal is to find a $\epsilon$-optimal policy $\pi$ after $K$ episodes. We say a policy $\pi$ is $\epsilon$-optimal if 
$$\mathbb{E}\left[\sum_{h=1}^{H} R_h\left(s_{h}, a_{h}\right) \mid \pi\right] \geq \mathbb{E}\left[\sum_{h=1}^{H} R_h\left(s_{h}, a_{h}\right) \mid \pi^{*}\right]-\epsilon,$$
where $\pi^*$ is the optimal policy for the MDP $(\mathcal{S}, \mathcal{A}, P,R,H,\nu)$.

\begin{theorem}
\label{thm: lower bound, online exploration}
Suppose $\epsilon \leq \min\left(C_1\sqrt{H}, C_2\sqrt{dH^4B}\right), B > 1, H > 4, d > 3$. Then for any algorithm $\mathcal{ALG}_1$ solving the non-reward-free online exploration problem, there exists an linear mixture MDP $\mathcal{M} = (\mathcal{S}, \mathcal{A}, P,R,H,\nu)$ such that $\mathcal{ALG}_1$ needs to collect at least $K = Cd^2H^3/\epsilon^2$ episodes to output an $\epsilon$-optimal policy for the linear mixture MDP $\mathcal{M}$ with probability at least $1-\delta$, where $C$ is an absolute constant, and $0 < \delta < 1$ is a positive constant that has no dependence on $\epsilon, H, d, K$.
\end{theorem}

From the above discussion, Theorem~\ref{thm: lower bound} can be directly proved by reduction if Theorem~\ref{thm: lower bound, online exploration} is true.

\begin{proof}
(Proof of Theorem~\ref{thm: lower bound}) The theorem can be proved by contradiction. Suppose there is an algorithm $\mathcal{ALG}$ which can $(\epsilon,\delta)$-learns any linear mixture MDP instance $\mathcal{M}$ with only $K' \leq C d^2H^3/\epsilon^2$ episodes. Then we can use this algorithm to solve the online exploration problem by simply ignoring the information about the reward function and directly calling the exploration algorithm of $\mathcal{ALG}$ during the exploration phase. Then in the planning phase, we use the planning algorithm of $\mathcal{ALG}$ to output a policy based on the reward function as well as samples collected in the exploration phase. Therefore, this indicates that $\mathcal{ALG}$ can output an $\epsilon$-optimal policy for the non-reward-free online exploration problem with probability at least $1-\delta$ after only $K' \leq C d^2H^3/\epsilon^2$ episodes. This contradicts the sample complexity lower bound in Theorem~\ref{thm: lower bound, online exploration}.
\end{proof}

Now we discuss on how to prove Theorem~\ref{thm: lower bound, online exploration}. 
\begin{proof}
(Proof of Theorem~\ref{thm: lower bound, online exploration})
Set $K_1 = cK$ for a positive constant $c \geq 2$. We construct another algorithm $\mathcal{ALG}_2$ for any possible $\mathcal{ALG}_1$ in the following way: In $\mathcal{ALG}_2$, the agent firstly runs $\mathcal{ALG}_1$ for $K_1/c = K$ episodes. After $K$ episodes, suppose $\mathcal{ALG}_1$ outputs a policy $\hat{\pi}$ according to certain policy distribution $\nu_{\pi}$.  $\mathcal{ALG}_2$ executes the policy $\hat{\pi}$ in the following $\frac{c-1}{c} K_1$ episodes. The interaction ends after $K_1$ episodes. $\mathcal{ALG}_2$ can be regarded as an algorithm which firstly runs the online exploration algorithm $\mathcal{ALG}_1$ for $K$ episodes, and then evaluates the performance of the policy $\hat{\pi}$ in the following episodes. We study the total regret of $\mathcal{ALG}_2$ from episode $K+1$ to episode $K_1$.

Recently, \cite{zhou2020nearly} proposed the following regret lower bound for linear mixture MDPs. In order to avoid confusion, we use the notation $K_2$ instead of $K$ to denote the number of total episodes of the regret minimization problem.
\begin{lemma}
\label{thm: lower bound, regret}
(Theorem 5.6 in \cite{zhou2020nearly})
 Let $B > 1$ and suppose $K_2 \geq \max\left\{(d-1)^2H/2,(d-1)/(32H(B-1))\right\}$, $d \geq 4$, $H \geq 3$. Then for any algorithm there exists an episodic, $B$-bounded linear mixture MDP parameterized by $\Theta = (\theta_1, \cdots, \theta_{H})$ such that the expected regret is at least $\Omega(dH\sqrt{HK_2})$.
\end{lemma}

To prove the above regret lower bound, \cite{zhou2020nearly} construct a class of hard instances which is an extension of the hard instance class for linear bandits (cf. Theorem 24.1 in \cite{lattimore2020bandit}). They show that for any algorithm, there exists a hard instance in the instance class such that the regret is at least $\Omega(dH\sqrt{HK_2})$. By Theorem~\ref{thm: lower bound, regret} and setting $K_2 = K_1$, we know that for any possible algorithm $\mathcal{ALG}_2$, there exists a hard instance $\mathcal{M}$ such that $\sum_{k=1}^{K_1} \mathbb{E}\left(V^*(s_1) - V^{\pi_k}(s_1)\right) \geq c_2 dH\sqrt{HK_1}$ for a positive constant $c_2$, where $\pi_k$ denotes the policy used in episode $k$ for algorithm $\mathcal{ALG}_2$. The expectation is over all randomness of the algorithm and the environment.

Note that in the hard instance constructed in \cite{zhou2020nearly}, the per-step regret is at most $\mathbb{E}\left(V^*(s_1) - V^{\pi_k}(s_1)\right) \leq \frac{dH}{4\sqrt{2}} \sqrt{\frac{H}{K_1}}$ for any episode $k \leq K_1$.  By choosing $c = \max\left\{\frac{1}{2\sqrt{2}c_2}, 2\right\}$, we know that for the instance $\mathcal{M}$, 
\begin{align}
    \sum_{k=K_1/c+1}^{K_1} \mathbb{E}\left(V^*(s_1) - V^{\pi_k}(s_1)\right) \geq \left(c_2 - \frac{1}{4\sqrt{2}c} \right) dH\sqrt{HK_1} \geq \frac{c_2}{2} dH\sqrt{HK_1}.
\end{align}

By the definition of $\mathcal{ALG}_2$, we have $\pi_k = \hat{\pi}$ if $k > K_1/c$. Therefore, we have
\begin{align}
    (c-1) K \mathbb{E}_{\hat{\pi} \sim \mu_{\pi}, s_1 \sim \mu}\left(V^*(s_1) - V^{\hat{\pi}}(s_1)\right) \geq \frac{c_2}{2} dH\sqrt{HK_1} = \frac{c_2c}{2} dH\sqrt{HK}.
\end{align}

Dividing both sides by $(c-1)K$, we have 
\begin{align}
    \mathbb{E}_{\hat{\pi} \sim \mu_{\pi}, s_1 \sim \mu}\left(V^*(s_1) - V^{\hat{\pi}}(s_1)\right) \geq  \frac{c_2 c}{2 (c-1)} dH\sqrt{H/K}. 
\end{align}

The above inequality indicates that, for any algorithm $\mathcal{ALG}_1$, there exists an instance $\mathcal{M}$ such that the expected sub-optimality gap of the policy returned by $\mathcal{ALG}_1$ is at least $\frac{c_2 c}{2 (c-1)} dH\sqrt{H/K}$ after collecting samples for $K$ episodes.

Suppose $\mathcal{ALG}_1$ returns an $\epsilon$-optimal policy $\hat{\pi}$ with probability at least $1-\delta$. Recall that the per-step sub-optimality gap is at most $\mathbb{E}\left(V^*(s_1) - V^{\pi_k}(s_1)\right) \leq \frac{dH}{4\sqrt{2}} \sqrt{\frac{H}{cK}}$ in the constructed hard instance. We have
\begin{align}
    (1-\delta) \cdot \epsilon + \delta \cdot \frac{dH}{4\sqrt{2}} \sqrt{\frac{H}{cK}} \geq \mathbb{E}_{\hat{\pi} \sim \mu_{\pi}, s_1 \sim \mu}\left(V^*(s_1) - V^{\hat{\pi}}(s_1)\right) \geq  \frac{c_2 c}{2 (c-1)} dH\sqrt{H/K}.
\end{align}
We set $\delta$ to be a constant satisfying $0 < \delta < \min\left\{1, \frac{2\sqrt{2}c^{1.5}c_2}{c-1}\right\}$. Solving the above inequality, we have $K \geq \frac{Cd^2H^3}{\epsilon^2}$ for a positive constant $C$.
 \end{proof}

\section{Improved Bound for Reward-free Exploration in Linear MDPs}
\label{appendix: sample complexity, linear MDP}
In this section, we explain how our choice of the exploration-driven reward can be used to improve the sample complexity bound for reward-free exploration in linear MDPs~\citep{wang2020reward}. We study the same reward-free setting as that in \cite{wang2020reward}, which is briefly explained in Appendix~\ref{appendix: preliminaries, linear MDPs}. We describe our algorithm and bound in Appendix~\ref{appendix: alg and thm for linear MDPs}, and prove our theorem in Appendix~\ref{appendix: proof for linear MDPs}.

\subsection{Preliminaries}
\label{appendix: preliminaries, linear MDPs}
For the completeness of explanation, we briefly restate the reward-free setting studied in \cite{wang2020reward}. Compared with the setting in this work, the main differences are twofold: Firstly, they study the linear MDPs setting instead of linear mixture MDPs in this work. Secondly, they study the standard reward-free exploration setting without the constraints of the plug-in solver.

The linear MDP assumption, which was first introduced in \cite{yang2019sample,jin2020provably}, states that the model of the MDP can be represented by linear functions of given features.
\begin{assumption}
\label{assumption: linear MDP}
(Linear MDP) an MDP $\mathcal{M} = (\mathcal{S}, \mathcal{A}, P, R, H, \nu)$ is said to be a linear MDP if the following hold:
\begin{itemize}
    \item There are $d$ unknown signed measures $\mu_h = (\mu_h^{(1)},\mu_h^{(2)}, \cdots, \mu_h^{(d)})$ such that for any $(s,a,s') \in \mathcal{S}\times \mathcal{A} \times \mathcal{S}$, $P_{h}\left(s^{\prime} \mid s, a\right)=\left\langle\mu_{h}\left(s^{\prime}\right), \phi(s, a)\right\rangle$.
    \item There exists $H$ unknown vectors $\eta_1, \eta_2, \cdots, \eta_H \in \mathbb{R}^d$ such that for any $(s,a) \in \mathcal{S} \times \mathcal{A}$, $R_h(s,a) = \left\langle\phi(s, a), \eta_{h}\right\rangle$.
\end{itemize}
We assume for all $(s,a) \in \mathcal{S} \times \mathcal{A}$ and $h \in [H]$, $\left\|\phi(s,a)\right\| \leq 1$, $\left\|\mu_h(s)\right\|_2 \leq \sqrt{d}$ and $\left\|\eta\right\|_2 \leq \sqrt{d}$.
\end{assumption}

For the reward-free exploration studied in \cite{wang2020reward}, the agent can collects a dataset of visited state-action pairs $\mathcal{D}=\left\{\left(s_{h}^{k}, a_{h}^{k}\right)\right\}_{(k, h) \in[K] \times[H]}$ which will be used in the planning phase. Then during the planning phase, the agent can follow a certain designed learning algorithm to calculate an $\epsilon$-optimal policy w.r.t. the reward function $R$ using the dataset $\mathcal{D}$.

\subsection{Algorithm and Theorem}
\label{appendix: alg and thm for linear MDPs}
Our algorithm can be divided into two parts. The exploration phase of the algorithm is presented in Algorithm~\ref{alg: exploration phase, linear MDPs}, and the planning phase is presented in Algorithm~\ref{alg: planning phase, linear MDPs}.

\begin{algorithm}
\caption{Reward-free Exploration for Linear MDPs: Exploration Phase}
\label{alg: exploration phase, linear MDPs}
  \begin{algorithmic}[5]
  \State Input: Failure probability $\delta>0$ and target accuracy $\epsilon > 0$
  \State $ \beta \leftarrow c_{\beta} \cdot d H \sqrt{\log \left(d H \delta^{-1} \varepsilon^{-1}\right)}$ for some $c_{\beta} > 0$
  \State $K \leftarrow c_{K} \cdot d^{3} H^{4} \log \left(d H \delta^{-1} \varepsilon^{-1}\right) / \varepsilon^{2}$ for some $c_K > 0$
    \For { episode $k = 1,2,\cdots, K$}
        \State ${Q}_{k,H+1}(\cdot,\cdot) \leftarrow 0$, ${V}_{k,H+1}(\cdot) \leftarrow 0$
        \For{step $h=H,H-1,\cdots, 1$}
            \State ${\Lambda}_{k,h} \leftarrow \sum_{t=1}^{k-1}  {\phi}(s_{t,h},a_{t,h}) {\phi}(s_{t,h},a_{t,h})^{\top} +  I$
            \State $u_{k,h}(\cdot, \cdot) \leftarrow \beta  \sqrt{ {\phi}(\cdot, \cdot)^{\top}\left( {\Lambda}_{k,h}\right)^{-1}  {\phi}(\cdot, \cdot)}$
            \State Define the exploration-driven reward function $R_{k,h}(\cdot,\cdot) = u_{k,h}(\cdot,\cdot)$
            \State $\hat{ {w}}_{k,h} \leftarrow \left( {\Lambda}_{k,h}\right)^{-1} \sum_{t=1}^{k-1}  {\phi}(s_{t,h},a_{t,h}) {V}_{t,h+1}(s_{t,h+1}) $
            \State $Q_{k,h}(\cdot, \cdot) \leftarrow \min \left\{\hat{ {w}}_{k,h}^{\top}  {\phi}(\cdot,\cdot)+R_{k,h}(\cdot, \cdot)+u_{k,h}(\cdot, \cdot), H\right\}$ 
            \State $V_{k,h}(s)\leftarrow \max_{a\in \mathcal{A}} Q_{k,h}(s,a)$, $\pi_{k,h}(s) = \argmax_{a \in \mathcal{A}} Q_{k,h}(s,a)$
        \EndFor
        \For{step $h = 1,2,\cdots, H$}
            \State Take action $a_{k,h} = \pi_{k,h}(s_{k,h})$ and observe $s_{k,h+1} \sim P_h(s_{k,h},a_{k,h})$
        \EndFor
    \EndFor
    \State Output: $\mathcal{D} \leftarrow\left\{\left(s_{k,h}, a_{k,h}\right)\right\}_{(k, h) \in[K] \times[H]}$.
  \end{algorithmic}
\end{algorithm}

\begin{algorithm}
\caption{Reward-free Exploration for Linear MDPs: Planning Phase}
\label{alg: planning phase, linear MDPs}
  \begin{algorithmic}[5]
  \State Input: Dataset $\mathcal{D} \leftarrow\left\{\left(s_{k,h}, a_{k,h}\right)\right\}_{(k, h) \in[K] \times[H]}$, reward functions $R = \{R_h\}_{h \in [H]}$
    \State ${Q}_{k,H+1}(\cdot,\cdot) = 0$, ${V}_{k,H+1}(\cdot) = 0$
    \For{step $h=H,H-1,\cdots, 1$}
        \State ${\Lambda}_{h} \leftarrow \sum_{t=1}^{K}  {\phi}(s_{t,h},a_{t,h}) {\phi}(s_{t,h},a_{t,h})^{\top} +  I$
         \State $u_{h}(\cdot, \cdot) \leftarrow \min \left\{ \beta  \sqrt{ {\phi}(\cdot, \cdot)^{\top}\left( {\Lambda}_{k,h}\right)^{-1}  {\phi}(\cdot, \cdot)}, H \right\}$
        \State $\hat{ {w}}_{h} \leftarrow \left( {\Lambda}_{h}\right)^{-1} \sum_{t=1}^{K}  {\phi}(s_{t,h},a_{t,h}) {V}_{t,h+1}(s_{t,h+1}) $
        \State $Q_{h}(\cdot, \cdot) \leftarrow \min \left\{\hat{ {w}}_{h}^{\top}  {\phi}(\cdot,\cdot)+R_{h}(\cdot, \cdot)+u_{h}(\cdot, \cdot), H\right\}$ 
        \State $V_{h}(s)\leftarrow \max_{a\in \mathcal{A}} Q_{h}(s,a)$, $\pi_{h}(s) = \argmax_{a \in \mathcal{A}} Q_{h}(s,a)$
    \EndFor
    \State Output: $\pi = \{\pi_h\}_{h \in [H]}$
  \end{algorithmic}
\end{algorithm}

Compared with the algorithm in \cite{wang2020reward}, the main difference is that we set $R_{k,h} = u_{k,h}$ in the exploration phase, instead of $R_{k,h} = u_{k,h}/H$. The following theorem states the complexity bound of our algorithms.

\begin{theorem}
\label{thm: sample complexity. linear MDP}
With probability at least $1-\delta$, after collecting $O(d^3H^4\log(dH\delta^{-1}\epsilon^{-1})/\epsilon^2)$ trajectories during the exploration phase, our algorithm outputs an $\epsilon$-optimal policy for any given reward during the planning phase.
\end{theorem}

\subsection{Proof of Theorem~\ref{thm: sample complexity. linear MDP}}
\label{appendix: proof for linear MDPs}
The proof of Theorem~\ref{thm: sample complexity. linear MDP} follows the proof framework in \cite{wang2020reward} with a slight modification. Therefore, we only sketch the proof and mainly focus on explaining the differences. Firstly, we introduce the value function $\tilde{V}^*_h(s,R)$, which is recursively defined from step $H+1$ to step $1$: 
\begin{align}
    \label{eqn: tildeV for linear MDP}
    &\tilde{V}^*_{H+1}(s,R) = 0, \forall s \in \mathcal{S} \\
    \tilde{V}^*_{h}(s,R) = & \max_{a \in \mathcal{A}} \left\{\min\left\{R_h(s,a) + P_h\tilde{V}^*_{h+1}(s,a,R),H\right\}\right\}, \forall s \in \mathcal{S}, h \in [H]
\end{align}
Compared with the definition of $V_h^*(s,R)$, the main difference is that we take minimization over the value and $H$ at each step. We can similarly define $\tilde{Q}^*_{h}(s,a,R)$, $\tilde{V}^{\pi}_{h}(s,R)$ and $\tilde{Q}^{\pi}_{h}(s,a,R)$.

\begin{lemma}
\label{lemma: to lemma 3.1}
With probability $1-\delta/2$, for all $k \in [K]$, 
\begin{align}
    \tilde{V}^*_{1}(s_{1,k},R_{k}) \leq V_{k,1}(s_{1,k})
\end{align}
and
\begin{align}
    \sum_{k=1}^{K} V_{k,1}(s_{k,1}) \leq c\sqrt{d^3H^4K \log(dKH/\delta)}
\end{align}
for some constant $c>0$ where $V_{k,1}$ is as defined in Algorithm~\ref{alg: exploration phase, linear MDPs}.
\end{lemma}
This lemma corresponds to Lemma~3.1 in \cite{wang2020reward}. The main difference is that we replace $V^*_1$ with $\tilde{V}^*_{1}$. Note that $\tilde{V}^*_{h}(s,R) \leq H$ by the definition of $\tilde{V}^*_{h}(s,R)$. Lemma~\ref{lemma: to lemma 3.1} can be similarly proved following the proof of Lemma~3.1 in \cite{wang2020reward} and replacing $V^*_h$ by $\tilde{V}^*_{h}$ during the proof.

\begin{lemma}
\label{lemma: to lemma 3.2}
With probability $1-\delta/4$, for the function $u_h(\cdot, \cdot)$ defined in Line 5 in Algorithm~\ref{alg: planning phase, linear MDPs}, we have
\begin{align}
    \mathbb{E}_{s \sim \mu}\left[\tilde{V}_{1}^{*}\left(s, u_{h} \right)\right] \leq c^{\prime} \sqrt{d^{3} H^{4} \cdot \log (d K H / \delta) / K}
\end{align}
\end{lemma}

Compared with Lemma~3.2 in \cite{wang2020reward}, we replace the term ${V}_{1}^{*}\left(s, u_{h}/H \right)$ with  $\tilde{V}_{1}^{*}\left(s, u_{h} \right)$. Note that with our choice of exploration-driven reward in the exploration phase (i.e. $R_{k,h} = u_{k,h}$), we replace the term $u_h/H$ with $u_h$ in the expectation. This lemma can be proved by following the proof of Lemma~3.2 in \cite{wang2020reward}.

\begin{lemma}
\label{lemma: to lemma 3.3}
With probability $1-\delta/2$, for any reward function satisfying Assumption~\ref{assumption: linear MDP} and all $h \in [H]$, we have 
\begin{align}
    Q_{h}^{*}(\cdot, \cdot, r) \leq Q_{h}(\cdot, \cdot) \leq R_{h}(\cdot, \cdot)+\sum_{s^{\prime}} P_{h}\left(s^{\prime} \mid \cdot, \cdot\right) V_{h+1}\left(s^{\prime}\right)+2 u_{h}(\cdot, \cdot).
\end{align}
\end{lemma}
Since our algorithm in the planning phase is exactly the same with that of \cite{wang2020reward}, this lemma shares the same idea of Lemma~3.3 in \cite{wang2020reward}.

Now we can prove Theorem~\ref{thm: sample complexity. linear MDP} with the help of Lemma~\ref{lemma: to lemma 3.2} and Lemma~\ref{lemma: to lemma 3.3}.

\begin{proof}
(Proof of Theorem~\ref{thm: sample complexity. linear MDP})
We condition on the events defined in Lemma~\ref{lemma: to lemma 3.2} and Lemma~\ref{lemma: to lemma 3.3}, which hold with probability at least $1-\delta$. By Lemma~\ref{lemma: to lemma 3.3}, we have for any $s\in \mathcal{S}$,
\begin{align}
    V_1(s) = \max_{s} Q_1(s,a) \geq \max_a Q^*_1(s,a,R) = V^*_1(s,R),
\end{align}
which implies 
\begin{align}
    \mathbb{E}_{s_{1} \sim \mu}\left[V_{1}^{*}\left(s_{1}, R\right)-V_{1}^{\pi}\left(s_{1}, r\right)\right] \leq \mathbb{E}_{s_{1} \sim \mu}\left[V_{1}\left(s_{1}\right)-V_{1}^{\pi}\left(s_{1}, R\right)\right].
\end{align}
Note that  $0 \leq V_h(s) \leq H$ and $0 \leq V^{\pi}_h(s,R) \leq H$ since $0 \leq R(s,a) \leq 1$. Therefore, we always have  $V_h(s) - V^{\pi}_h(s,R) \leq H$. For any $s_h$, we have
\begin{align}
    &V_h(s_h) - V_h^{\pi}(s_h,R) \\
    \leq & \min\left\{H,  P_h V_{h+1}(s_h,\pi_h(s_h)) + 2u_h(s_h,\pi_h(s_h)) -  P_h V_{h+1}^{\pi}(s_{h+1},R)\right\} 
\end{align}
By recursively decomposing $V_h(s_h) - V_h^{\pi}(s_h,R)$ from step $H$ to $1$, we have
\begin{align}
    \mathbb{E}_{s_{1} \sim \mu}\left[V_{1}\left(s_{1}\right)-V_{1}^{\pi}\left(s_{1}, r\right)\right] \leq \mathbb{E}_{s \sim \mu}\left[\tilde{V}_{1}^{\pi}(s, u)\right].
\end{align}
By definition of $\tilde{V}_1^*(s,u)$, we have $\mathbb{E}_{s \sim \mu}\left[\tilde{V}_{1}^{\pi}(s, u)\right] \leq 2\mathbb{E}_{s \sim \mu}\left[\tilde{V}_{1}^{*}(s, u)\right].$ By Lemma~\ref{lemma: to lemma 3.2}, 
\begin{align}
    \mathbb{E}_{s \sim \mu}\left[\tilde{V}_{1}^{*}(s, u)\right] \leq c^{\prime}  \sqrt{d^{3} H^{4} \cdot \log (d K H / \delta) / K}.
\end{align}

By taking $K = c_K d^3H^4 \log(dH\delta^{-1}\epsilon^{-1})/\epsilon^2$ for a sufficiently large constant $c_K > 0$, we have 
\begin{align}
    \mathbb{E}_{s_{1} \sim \mu}\left[V_{1}^{*}\left(s_{1}, R\right)-V_{1}^{\pi}\left(s_{1}, r\right)\right] \leq  c^{\prime} \sqrt{d^{3} H^{4} \cdot \log (d K H / \delta) / K} \leq \epsilon.
\end{align}
\end{proof}

\section{Auxiliary Lemmas}
\label{appendix: auxiliary lemmas}
\begin{lemma}
\label{lemma: self-normalized bound}
(Self-Normalized Bound for Vector-Valued Martingales, Theorem 1 and 2 in~\cite{abbasi2011improved}) Let $\{F_t\}_{t=0}^{\infty}$ be a filtration. Let $\{\eta_t\}_{t=1}^{\infty}$ be a real-valued stochatic process such that $\eta_t$ is $F_t$-measurable and $\eta_t$ is conditionally $R$-sub-Gaussian for some $R \geq 0$, i.e.
\begin{align}
    \forall \lambda \in \mathbb{R}, \mathbb{E} [e^{\lambda \eta_t}|F_{t-1}] \leq \exp{\left(\frac{\lambda^2R^2}{2}\right)}.
\end{align}
Let $\{X_t\}_{t=1}^{\infty}$ be an $\mathbb{R}^d$-valued stochastic process such that $X_t$ is $F_{t-1}$-measurable. Assume that $V$ is a $d\times d$ positive definite matrix. For any $t \geq 0$, define
\begin{align}
    \bar{V}_t = V + \sum_{s=1}^{t} X_s X_s^{\top}, S_t = \sum_{s=1}^{t} \eta_s X_s.
\end{align}
Then, for any $\delta > 0$, with probability at least $1-\delta$, for all $t \geq 0$,
\begin{align}
    \left\|S_{t}\right\|_{\bar{V}_{t}^{-1}}^{2} \leq 2 R^{2} \log \left(\frac{\operatorname{det}\left(\bar{V}_{t}\right)^{1 / 2} \operatorname{det}(V)^{-1 / 2}}{\delta}\right).
\end{align}

Further, let $V = I \lambda, \lambda > 0$. Define $Y_{t}=\left\langle X_{t}, \theta_{*}\right\rangle+\eta_{t}$ and assume that $\|\theta_{*}\|_2 \leq S$, $\left\|X_{t}\right\|_{2} \leq L,\forall t \geq 1$. Then for any $\delta >0$,with probability at least $1-\delta$, for all $t \geq 0$, $\theta_{*}$ satisfies
\begin{align}
    \left\|\hat{\theta}_{t}-\theta_{*}\right\|_{\bar{V}_{t}} \leq R \sqrt{d \log \left(\frac{1+t L^{2} / \lambda}{\delta}\right)}+\lambda^{1 / 2} S,
\end{align}
where $\hat{\theta}_t$ is the $l^2$-regularized least-squares estimation of $\theta_{*}$ with regularization parameter $\lambda > 0$ based on history samples till step $t$.
\end{lemma}

\begin{lemma}
\label{lemma: data accumulation}
(Lemma 11 in~\cite{abbasi2011improved}) Let $\{X_t\}_{t=1}^{\infty}$ be a sequence in $\mathbb{R}^d$, $V$ a $d\times d$ positive definite matrix and define $\bar{V}_t = V + \sum_{s=1}^{t} X_sX_s^{\top}$. Then, we have that 
\begin{align}
    \log \left(\frac{\operatorname{det}\left(\bar{V}_{n}\right)}{\operatorname{det}(V)}\right) \leq \sum_{t=1}^{n}\left\|X_{t}\right\|_{\bar{V}_{t-1}^{-1}}^{2}.
\end{align}
Further , if $\left\|X_{t}\right\|_{2} \leq L$ for all $t$, then 
\begin{align}
    \sum_{t=1}^{n} \min \left\{1,\left\|X_{t}\right\|_{\bar{V}_{t-1}^{-1}}^{2}\right\} \leq 2\left(\log \operatorname{det}\left(\bar{V}_{n}\right)-\log \operatorname{det} V\right) \leq 2\left(d \log \left(\left(\operatorname{trace}(V)+n L^{2}\right) / d\right)-\log \operatorname{det} V\right),
\end{align}
and finally, if $\lambda_{\min}(V) \geq \max(1,L^2)$, then
\begin{align}
    \sum_{t=1}^{n}\left\|X_{t}\right\|_{\bar{V}_{t-1}^{-1}}^{2} \leq 2 \log \frac{\operatorname{det}\left(\bar{V}_{n}\right)}{\operatorname{det}(V)}.
\end{align}
\end{lemma}

\begin{lemma}
\label{lemma: self-normalized bound, Bernstein}
(Bernstein inequality for vector-valued martingales, Theorem 4.1 in~\cite{zhou2020nearly}) Let $\left\{\mathcal{G}_{t}\right\}_{t=1}^{\infty}$ be a filtration, $\{x_t,\eta_t\}_{t \geq 1}$ a stochastic process so that $x_t \in \mathbb{R}^d$ is $\mathcal{G}_t$-measurable and $\eta_t \in \mathbb{R}$ is $\mathcal{G}_{t+1}$-measurable. Fix $R,L,\sigma,\lambda > 0, \mu^* \in \mathbb{R}^d$. For $t \geq 1$, let $y_t =\left\langle{\mu}^{*}, {x}_{t}\right\rangle+\eta_{t}$ and suppose that $\eta_t, x_t$ also satisfy
\begin{align}
    \left|\eta_{t}\right| \leq R, \mathbb{E}\left[\eta_{t} \mid \mathcal{G}_{t}\right]=0, \mathbb{E}\left[\eta_{t}^{2} \mid \mathcal{G}_{t}\right] \leq \sigma^{2},\left\|{x}_{t}\right\|_{2} \leq L.
\end{align}
Then, for any $0 < \delta < 1$, with probability at least $1-\delta$, we have
\begin{align}
    \forall t>0,\left\|\sum_{i=1}^{t} {x}_{i} \eta_{i}\right\|_{{z}_{t}^{-1}} \leq \beta_{t},\left\|{\mu}_{t}-{\mu}^{*}\right\|_{{z}_{t}} \leq \beta_{t}+\sqrt{\lambda}\left\|{\mu}^{*}\right\|_{2},
\end{align}
where for $t \geq 1, {\mu}_{t}={Z}_{t}^{-1} {b}_{t}, {Z}_{t}=\lambda {I}+\sum_{i=1}^{t} {x}_{i} {x}_{i}^{\top}, {b}_{t}=\sum_{i=1}^{t} y_{i} {x}_{i}$ and 
\begin{align}
    \beta_{t}=8 \sigma \sqrt{d \log \left(1+t L^{2} /(d \lambda)\right) \log \left(4 t^{2} / \delta\right)}+4 R \log \left(4 t^{2} / \delta\right).
\end{align}

\end{lemma}

\end{document}